%% file: main.tex
\documentclass{article} 
\usepackage[table,RGB]{xcolor}
\usepackage{sections/iclr2026_conference,times}

\usepackage[utf8]{inputenc} 
\usepackage[T1]{fontenc}    
\usepackage{url}            
\usepackage{booktabs}       
\usepackage{amsfonts}       
\usepackage{nicefrac}       
\usepackage{microtype}      
\usepackage{mathtools}
\usepackage{wrapfig}
\usepackage{bm}
\usepackage{float}
\usepackage{amsthm}
\usepackage{amsmath}
\usepackage{amssymb}
\usepackage{booktabs}
\usepackage{graphicx}
\usepackage{graphbox}
\usepackage{notes2bib}
\usepackage{multirow}
\usepackage{algorithm, algorithmic}
\usepackage{setspace}
\usepackage{makecell}
\usepackage{threeparttable}

\usepackage{tcolorbox}
\definecolor{darkerlogocolor}{RGB}{20, 0, 145}  
\newtcolorbox{ttcolorbox}[1][]{colframe=darkerlogocolor, colback=darkerlogocolor!4!white, title=#1}

\usepackage[colorlinks,linkcolor=blue]{hyperref}

\usepackage[capitalize,noabbrev]{cleveref}
\crefname{section}{Section}{Sections}
\crefname{section}{Appendix}{Appendices}
\crefname{theorem}{Theorem}{Theorems}
\crefname{lemma}{Lemma}{Lemmas}
\crefname{equation}{Equation}{Equations}
\crefname{proposition}{Proposition}{Propositions}
\crefname{claim}{Claim}{Claims}
\crefname{appendix}{Appendix}{Appendices}
\crefname{algorithm}{Algorithm}{Algorithms}
\crefname{figure}{Figure}{Figures}
\crefname{table}{Table}{Tables}
\crefname{remark}{Remark}{Remarks}
\crefname{definition}{Def.}{Definitions}
\crefname{corollary}{Corollary}{Corollaries}

\input{sections/def_global}

\usepackage{etoc}
\etocdepthtag.toc{mtchapter}
\etocsettagdepth{mtchapter}{subsection}
\etocsettagdepth{mtappendix}{none}

\usepackage{xspace}

\def \pref{p_{\mathrm{ref}}}
\def \ptilt{p_{\mathrm{tilt}}}

\title{Diffusion Alignment Beyond KL: Variance Minimisation as Effective Policy Optimiser}

\author{
Zijing Ou$^{1}$\thanks{Work done partially during the internship at Samsung R\&D Institute UK.}, Jacob Si$^{1}$\footnotemark[1], Junyi Zhu$^{2}$, Ondrej Bohdal$^{2}$, Mete Ozay$^{2}$, Taha Ceritli$^{2}$, Yingzhen Li$^{1}$ \\
$^1$Imperial College London, \ $^2$Samsung R\&D Institute UK \\
\texttt{z.ou22@imperial.ac.uk} 
}

\iclrfinalcopy 
\begin{document}

\maketitle

\vspace{-2mm}
\begin{abstract}
Diffusion alignment adapts pretrained diffusion models to sample from reward-tilted distributions along the denoising trajectory. This process naturally admits a Sequential Monte Carlo (SMC) interpretation, where the denoising model acts as a proposal and reward guidance induces importance weights.
Motivated by this view, we introduce Variance Minimisation Policy Optimisation (VMPO), which formulates diffusion alignment as minimising the variance of log importance weights rather than directly optimising a Kullback-Leibler (KL) based objective.
We prove that the variance objective is minimised by the reward-tilted target distribution and that, under on-policy sampling, its gradient coincides with that of standard KL-based alignment.
This perspective offers a common lens for understanding diffusion alignment.
Under different choices of potential functions and variance minimisation strategies, VMPO recovers various existing methods, while also suggesting new design directions beyond KL.
\end{abstract}

\input{sections/01-intro}
\input{sections/02-method}
\input{sections/03-experiment}

\input{sections/04-conclusion}

{
\bibliography{main}
\bibliographystyle{sections/icml2022}
}

\input{sections/07-appendix}

\end{document}

%% file: sections/def_global.tex
\usepackage{bm}
\usepackage{float}
\usepackage{amsthm}
\usepackage{amsmath}
\usepackage{amssymb}
\usepackage{booktabs}
\usepackage{graphicx}
\usepackage{graphbox}
\usepackage{notes2bib}
\usepackage{subcaption}
\usepackage{multirow}

\usepackage[colorlinks,linkcolor=blue]{hyperref}
\usepackage[dvipsnames]{xcolor}
\definecolor{cite_color}{HTML}{114083}
\definecolor{link_color}{RGB}{0,102,102}
\definecolor{link_color}{RGB}{153, 0,0}  
\definecolor{url_color}{RGB}{153, 102,  0}
\definecolor{emp_color}{RGB}{0,0,255}
\hypersetup{
 colorlinks,
 citecolor=cite_color,
 linkcolor=link_color,
 urlcolor=url_color}
\graphicspath{{./figs/}}

\def\E{{\mathbb E}}

\newcommand*\dif{\mathop{}\!\mathrm{d}}

\def \exp{\mathrm{exp}}

\renewcommand{\mid}{|}

\newtheorem{innercustomgeneric}{\customgenericname}
\providecommand{\customgenericname}{}
\newcommand{\newcustomtheorem}[2]{%
  \newenvironment{#1}[1]
  {%
   \renewcommand\customgenericname{#2}%
   \renewcommand\theinnercustomgeneric{##1}%
   \innercustomgeneric
  }
  {\endinnercustomgeneric}
}

\newcustomtheorem{customthm}{Theorem}

\DeclareMathOperator*{\argmax}{argmax}
\DeclareMathOperator*{\argmin}{argmin}

\usepackage{thmtools}
\usepackage{thm-restate}

\usepackage{tikz}
\usetikzlibrary{decorations.pathreplacing,calc}

\usepackage{tikz}
\newcommand{\canceltozero}[1]{%
  \tikz[baseline=(N.base)]{
    \node[inner sep=1pt, outer sep=1pt] (N) {$#1$};
    \draw[->, thin] (N.south west) -- (N.north east) node[right, scale=0.8] {0};
  }%
}

%% file: sections/01-intro.tex
\section{Introduction}
Diffusion models \citep{ho2020denoising,song2020score,rombach2022high} have recently moved beyond large-scale pretraining to incorporate downstream objectives, such as aligning with human preferences \citep{wallace2024diffusion}. 
A prominent line of work, commonly referred to as diffusion alignment, formalises this idea by adapting pretrained diffusion models such that their generated samples are biased toward high-reward regions of the data space \citep{black2023training,fan2023dpok,liu2025flow,xue2025dancegrpo}. 
This direction has recently attracted significant attention, as alignment with reward signals enables controllable generation at inference time, allowing pretrained denoising models to be steered toward desired behaviours without retraining from scratch or compromising sample quality \citep{clark2023directly,domingo2024adjoint,potaptchik2025tilt,potaptchik2026meta}.

Existing approaches, including Denoising Diffusion Policy Optimisation (DDPO) \citep{black2023training}, Diffusion policy optimisation with KL regularisation (DPOK) \citep{fan2023dpok}, and flow-based Group Relative Policy Optimisation (Flow-GRPO) \citep{liu2025flow}, typically cast diffusion alignment as a reinforcement learning problem \citep{schulman2015trust,schulman2017proximal,shao2024deepseekmath}, applying KL-regularised policy updates along the denoising trajectory. While these methods have demonstrated strong empirical performance, they are often derived from specific optimisation viewpoints, which can obscure their underlying connections and make it difficult to systematically design new variants. In addition, these formulations emphasise matching reward-tilted target distributions via KL minimisation, while alternative optimisation criteria remain relatively underexplored.

In this work, we revisit diffusion alignment through the lens of Sequential Monte Carlo (SMC) \citep{del2006sequential}. By interpreting the denoising process as a proposal distribution and reward guidance as inducing importance weights, we show that diffusion alignment naturally gives rise to a variance minimisation problem \citep{richter2020vargrad}.
Motivated by this observation, we propose Variance Minimisation Policy Optimisation (VMPO), which frames alignment as minimising the variance of log importance weights along the trajectory. We analyse its theoretical properties and show that, under on-policy sampling, its gradient recovers that of standard KL-based alignment.
Beyond providing a principled interpretation, this perspective offers a flexible framework: different choices of potential functions and variance objectives lead to optimisation rules closely related to existing methods, while also suggesting new design directions beyond KL regularisation. 
Empirically, we demonstrate the effectiveness of VMPO by finetuning Stable Diffusion 1.5 and 3.5 \citep{rombach2022high,esser2024scaling} across a wide range of reward functions.
We hope that this variance-minimisation viewpoint provides a useful conceptual tool for understanding diffusion alignment and for guiding the development of future alignment methods.

%% file: sections/02-method.tex
\section{Background: Diffusion Alignment}

Given a pretrained diffusion model $p_{\mathrm{ref}}(x_{t-1}|x_t)$ and a reward model $r$, diffusion alignment aims at drawing samples from the reward-tilted target $\ptilt (x_{t\!-\!1} | x_t) \!\propto\! p_{\mathrm{ref}}(x_{t\!-\!1}|x_t)\exp(\!\frac{r(x_{t\!-\!1})}{\beta}\!)$
\footnote{We omit the condition $c$ when clear from context. In most settings, the reward function is defined only on the clean sample $x_0$. We defer the discussion of how to construct intermediate reward $r(x_t)$ to \cref{sec:appendix-vmpo-kaleidoscopes}}.
A naive way to sample from $\ptilt$ is to apply self-normalised importance sampling (SNIS) \citep{owen2013monte}:
\begin{align}
    \!\!\!\!x_{t-1} \!\sim\! \mathrm{Cat} \!\left(\!x^i; \frac{w^i_t}{\sum_i w^i_t}\right)\!,\ \log w^i_t \!=\! \log \frac{p_{\mathrm{ref}}(x^i_{t-1} | x_t^i)\exp(r(x^i_{t-1})/\beta)}{p_\theta (x^i_{t-1} | x_t^i)}\!,\ x^i_{t-1} \!\sim\! p_\theta(x^i_{t-1} | x_t^i). \!\!
\end{align}
While straightforward, SNIS suffers from high variance with a poor proposal $p_\theta$, which precipitates particle degeneracy and hinders adequate exploration of the state space \citep{del2006sequential}. Theoretically, the optimal proposal with zero variance is $p_\theta (x_{t-1}|x_t) \propto p_{\mathrm{ref}}(x_{t-1} | x_t)\exp(r(x_{t-1})/\beta)$, which inspires us to learn an optimal proposal by minimising a divergence between $p_\theta$ and $\ptilt$. A popular choice is the KL-divergence \citep{kullback1951information} (see \cref{sec:diffusion-alignment-prob-inf} for the derivation):
\begin{align} \label{eq:kl-reward-align}
    \argmin_\theta \mathbb{KL}\left(p_\theta || \ptilt \right) 
    \Leftrightarrow 
    \argmax_\theta \E_{p_\theta} [r(x_{t-1})] - \beta \mathbb{KL}(p_\theta (x_{t-1} | x_t) || p_{\mathrm{ref}}(x_{t-1} | x_t)).
\end{align}
To optimise $p_\theta$ using \cref{eq:kl-reward-align}, one can apply the policy gradient \citep{sutton1999policy}, which induces the following surrogate objective:
\begin{align} \label{eq:kl-policy-gradient}
    \mathcal{J}_{\mathrm{KL}}(t; \theta) = \E_{p_{\theta_\mathrm{old}}}\left[\frac{p_\theta(x_{t-1} | x_t)}{p_{\theta_\mathrm{old}} (x_{t-1} | x_t)} r(x_{t-1})\right] - \beta \mathbb{KL}(p_\theta (x_{t-1} | x_t) || p_{\mathrm{ref}}(x_{t-1} | x_t)).
\end{align}
\cref{eq:kl-policy-gradient} has been widely adopted in reinforcement learning algorithms, such as proximal policy optimisation (PPO) \citep{schulman2017proximal} and GRPO \citep{shao2024deepseekmath}, which have recently been adapted in diffusion alignment, such as DDPO \citep{black2023training} and Flow-GRPO \citep{liu2025flow} (see \cref{sec:appendix-align-two-view} for a brief review). Rather than following the KL perspective, in the following, we propose a new diffusion alignment method from the perspective of variance minimisation.

\section{VMPO: Variance Minimisation Policy Optimiser} \label{sec:vmpo}
In this section, we introduce VMPO: a policy optimiser with variance minimisation. 
Rather than directly maximising expected reward, VMPO learns the policy $p_\theta$ by minimising the variance of importance weights \citep{richter2020vargrad}
\footnote{For ease of exposition, we focus only on consecutive timesteps $(t-1, t)$. In \cref{sec:appendix-demystify-vmpo}, we justify the variance minimisation objective from the lens of SMC, and further show that $T \sum_{t=1}^T \mathcal{L}_h^{\mathrm{Var}}(t;\theta)$ upper-bounds $\frac{1}{2}\mathbb{V}_h(\log \frac{p_{\mathrm{ref}}(x_{0:T})}{p_\theta (x_{0:T})} + \frac{1}{\beta} \sum_{t=0}^T r(x_t))$, i.e., the sum of local variance upper-bounds the variance of the importance weight of the joint trajectory distribution.}
\begin{align} \label{eq:log-var-loss}
    \!\!\!\!\!\!\mathcal{L}^{h}_{\mathrm{Var}} \!(t;\theta) \!\!=\!\! \frac{1}{2} \!\mathbb{V}_{\!h} \! ( \log \! w_t ) \!=\!\frac{1}{2} \E_h \!\! \left[\! \log\! \frac{p_{\mathrm{ref}}(x_{t\!-\!1} | x_t) \exp(\frac{r(x_{t\!-\!1})}{\beta})}{p_\theta (x_{t\!-\!1} | x_t)} \!\!-\! \E_h \!\!\left[\! \log\! \frac{p_{\mathrm{ref}}(x_{t-1} | x_t) \exp(\frac{r(x_{t\!-\!1})}{\beta})}{p_\theta (x_{t-1} | x_t)}\!\right]\! \right]^2\!\!\!\!\!
\end{align}
where $h$ is an arbitrary reference distribution sharing the same support as $p_\theta$ and $p_{\mathrm{ref}}$. The validity of the loss \cref{eq:log-var-loss} is formalised by the following proposition:
\begin{restatable}[]{proposition}{restapropone} \label{prop:optimal-proposal}
    The optimum of $\mathcal{L}^{h}_{\mathrm{Var}}(t;\theta)$ satisfies $p_{\theta^*} = \ptilt(x_{t-1} | x_t)$, $\theta^* = \argmin_\theta \mathcal{L}^{h}_{\mathrm{Var}}(t;\theta)$. Moreover, $ \left. \nabla_\theta \mathcal{L}^{h}_{\mathrm{Var}}(t;\theta) \right|_{h=p_\theta} = \nabla_\theta \mathbb{KL}(p_\theta (x_{t-1} | x_t) || \ptilt(x_{t-1} | x_t))$.
\end{restatable}
In practice, $\mathcal{L}^{h}_{\mathrm{Var}}$ can be estimated using Monte Carlo samples:
\begin{align} \label{eq:log-var-monte-carlo}
    \hat{\mathcal{L}}^{h}_{\mathrm{Var}}(t;\theta) \!:=\! \frac{1}{2(K-1)} \sum_{i=1}^K \left( 
    \!\underbrace{\!
        \log R_\theta (x^i_{t-1} | x_t^i) \!+\! \frac{1}{\beta} r(x^i_{t-1}) \!-\! \overline{\log R_\theta (x_{t-1} | x_t)} \!-\! \frac{1}{\beta} \overline{r(x_{t-1})}
    }_{:= A_t^i (\theta)} 
    \!\right)^2\!,
\end{align}
where $x^i_t \sim h$, $\overline{\log R_\theta (x_{t-1} | x_t)} = \frac{1}{N} \sum_{i=1}^K \log \frac{p_{\mathrm{ref}}(x_{t-1} | x_t))}{p_\theta (x_{t-1} | x_t)}$, and $\overline{r(x_{t-1})} = \frac{1}{K} \sum_{i=1}^K r(x^i_{t-1})$.
It is noteworthy that although $\hat{\mathcal{L}}^{h}_{\mathrm{Var}}(t;\theta)$ is a biased estimator of $\mathcal{L}^h_{\mathrm{Var}}(t;\theta)$ due to the quadratic nonlinearity of $(\cdot)^2$, its gradient, which has the form of
\begin{align} \label{eq:graidnet-monte-carlo-gradient}
    \nabla_\theta \hat{\mathcal{L}}^{h}_{\mathrm{Var}}(t;\theta) 
    = \frac{1}{(K-1)} \sum_{i=1}^K - A_t^i (\theta)  \nabla_\theta \log p_\theta (x_{t-1} | x_t),
\end{align}
is an unbiased estimator to $- \nabla_\theta \mathcal{J}_{\mathrm{KL}}(t;\theta)$ if we set $h=p_\theta$.
\cref{eq:log-var-monte-carlo} induces a loss to train the policy: $\argmin_\theta \frac{1}{(N-1)} \sum_{i=1}^K \sum_{t=1}^T - A_t^i (\theta_{\mathrm{sg}}) \log p_\theta (x_{t-1} | x_t)$, which closely resembles the GRPO objective (see \cref{sec:proof-unbias-gradient} for details).
Although straightforward, simple Monte Carlo estimation of the expectation $\E_h [\log w_t]$ typically has high variance, which in turn requires a high number of Monte Carlo samples to improve training stability, but reduces the training efficiency.
Alternatively, we can instead amortise this expectation using a learnable function approximator. Concretely, we introduce a neural estimator $M_\phi(t)$ trained by
\begin{align}
    \phi^* = \argmin_\phi \E_h \left[ \left(\log R_\theta (x_{t-1} | x_t) \!+\! \frac{1}{\beta} r(x_{t-1}) - M_\phi (t) \right)^2 \right].
\end{align}
By the identity $\E_h [w] = \argmin_M \E_h [(w - M)^2]$, the optimal satisfies $M_{\phi^*}(t) = \E_h [\log w_t]$. Substituting this estimator into the variance objective yields the VMPO loss
\begin{align} \label{eq:vmpo-amor-z-obj}
    \mathcal{L}_{\mathrm{VMPO}} (\theta, \phi) = \E_{h, t} \left[ \left(\log R_\theta (x_{t-1} | x_t) \!+\! \frac{1}{\beta} r(x_{t-1}) - M_\phi (t) \right)^2 \right]
\end{align}
\begin{wrapfigure}{r}{0.45\linewidth}
\vspace{-4mm}
\centering
    \vspace{-4mm}
    \begin{minipage}[t]{0.99\linewidth}
    \centering
    \begin{algorithm}[H]
    \caption{VMPO} \small
    \label{alg:vmpo} 
    \setstretch{1.05}
    \begin{algorithmic}[1] 
        \REQUIRE Reference model $\pref$, denoising policy $p_\theta$, mean estimator $M_\phi$, reward function $r$
        \WHILE{not converged}
            \STATE Rollout $\tau = \{x_t\}_t$ with $p_\theta (x_{t-1} | x_t)$
            \STATE For each transition $(x_{t-1}, x_t) \in \tau$: \\update $\theta, \phi$ with \cref{eq:vmpo-amor-z-obj}
        \ENDWHILE
    \end{algorithmic}
    \end{algorithm}
    \end{minipage}
\label{fig:algorithms}
\end{wrapfigure}
We summarise the training procedure in \cref{alg:vmpo}.
For the choice of intermediate reward $r(x_t)$, we consider two options: return-to-go and difference,  in \cref{sec:appendix-vmpo-kaleidoscopes}, which lead to two variants of our method: VMPO-R2G and VMPO-Diff, respectively.
Moreover, in \cref{sec:appendix-conect-related-work}, we further show that VMPO serves as a kaleidoscopic policy optimiser, connecting to a broad class of existing diffusion alignment methods through different choices of variance minimisation strategies.

%% file: sections/03-experiment.tex
\begin{figure}[!t]
    \centering
    \begin{minipage}{0.99\linewidth}
        \centering
        {\scriptsize ---\ Anthropomorphic Virginia opossum playing guitar. --->}
        \includegraphics[width=.99\linewidth]{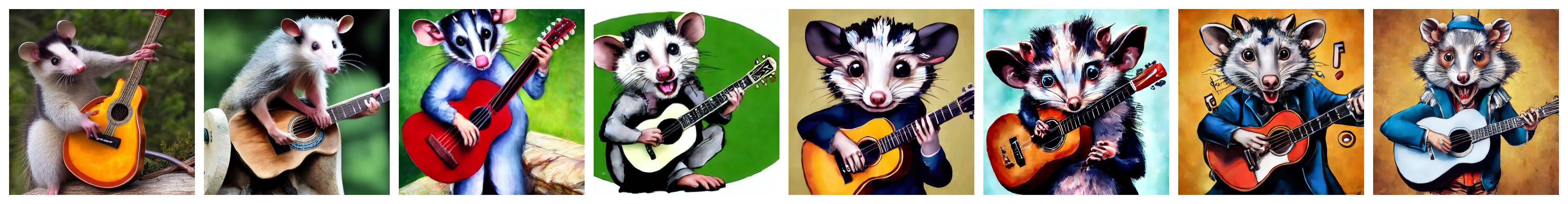}
    \end{minipage}
    \begin{minipage}{0.99\linewidth}
        \centering
        {\scriptsize ---\ The image features a castle surrounded by a dreamy garden with roses and a cloudy sky in the background. --->}
        \includegraphics[width=.99\linewidth]{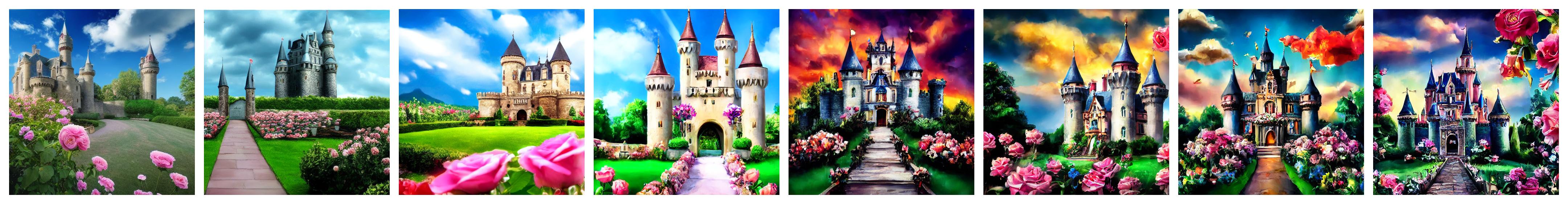}
    \end{minipage}
    \vspace{-2mm}
    \caption{Visualisation of alignment dynamics over the training progress of SD1.5 with HPSv2. The generated images become more faithful to the prompt as the training continues (from left to right).}
    \label{fig:t2i-hpsv2-training}
    \vspace{-4mm}
\end{figure}

\section{Experiments}

We evaluate VMPO by aligning Stable Diffusion v1.5 \citep{rombach2022high} using Human Preference Score (HPSv2) \citep{wu2023human}. Details of the experimental setup, along with additional results on Stable Diffusion v3.5 \citep{esser2024scaling} and other reward functions, are deferred to \cref{sec:appendix-additional-exp}.

\textbf{Setup.}
For training, we adopt low-rank adaptation (LoRA) \citep{hu2022lora} and leverage the photo and painting prompts from \cite{wu2023human} as training prompts.
For testing, we consider the 100 test prompts\footnote{\scriptsize \url{https://github.com/microsoft/soc-fine-tuning-sd/blob/main/prompt_files/benchmark_ir.json}} collected by \cite{domingo2024adjoint}. For each prompt, we generate 10 images and evaluate the CLIPScore \citep{hessel2021clipscore}, ImageReward \citep{xu2023imagereward}, and DreamSim \citep{fu2023dreamsim} as the reported metrics.

\begin{table}[!t] 
    \centering
    \caption{Results on Stable Diffusion v1.5 fine-tuned with HPSv2.}
    \label{tab:hps-results}
    \vspace{-3mm}
    \begin{tabular}{lcccc}
        \toprule 
        \textbf{Method} & \textbf{HPSv2} ($\uparrow$) & \textbf{CLIPScore} ($\uparrow$) & \textbf{ImageReward} ($\uparrow$) & \textbf{DreamSim} ($\uparrow$) \\
        \midrule 
        SD1.5 (Base) &
        0.2368 ± 0.0029 & 0.2717 ± 0.0032	& 0.0331 ± 0.0779 & 0.4389 ± 0.0116 \\
        GRPO &  
        0.2684 ± 0.0035 & 0.2653 ± 0.0034	   & 0.3449 ± 0.0758	 & 0.3220 ± 0.0098 \\
        VMPO-R2G &  
        0.2723 ± 0.0032 & 0.2713 ± 0.0030	   & 0.3427 ± 0.0762	 & 0.3673 ± 0.0115 \\
        VMPO-Diff &  
        0.2822 ± 0.0040 & 0.2622 ± 0.0028	   & 0.4973 ± 0.0780	 & 0.2916 ± 0.0104 \\
        \bottomrule 
    \end{tabular}
\end{table}

\begin{figure}[!t]
    \vspace{-3mm}
    \centering
    \begin{minipage}{.99\linewidth}
        \centering
        \begin{minipage}{.245\linewidth}\centering SD-1.5 \end{minipage}
        \hfill
        \begin{minipage}{.245\linewidth}\centering GRPO \end{minipage}
        \hfill
        \begin{minipage}{.245\linewidth}\centering VMPO-R2G \end{minipage}
        \hfill
        \begin{minipage}{.245\linewidth}\centering VMPO-Diff \end{minipage}
    \end{minipage}
    \begin{minipage}{0.99\linewidth}
        \centering
        \begin{minipage}{0.245\linewidth}
            \centering
            \includegraphics[width=.99\linewidth]{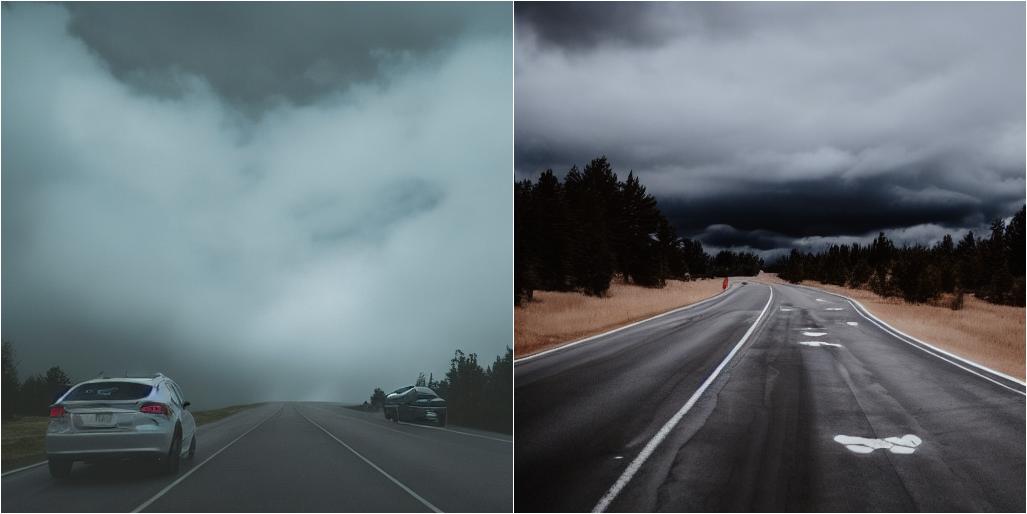}
        \end{minipage}
        \hfill
        \begin{minipage}{0.245\linewidth}
            \centering
            \includegraphics[width=.99\linewidth]{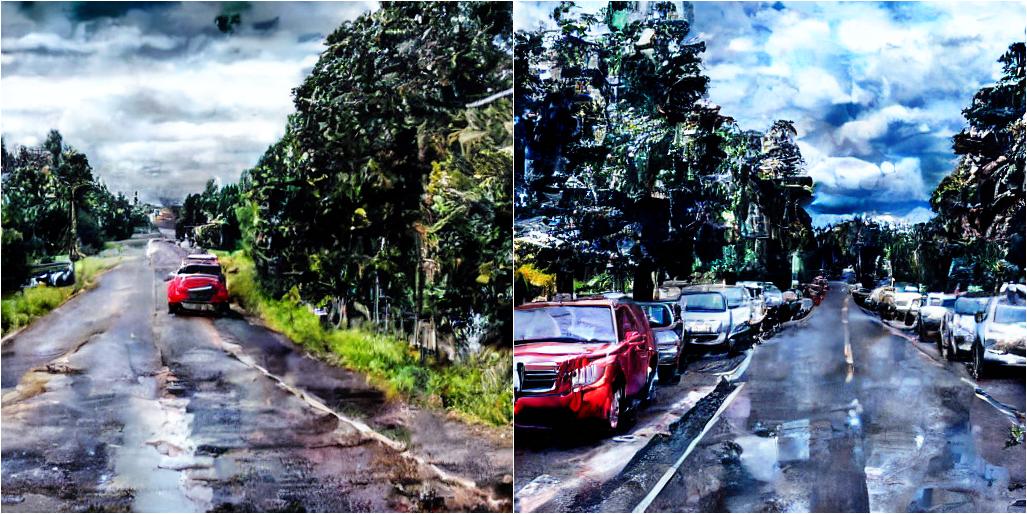}
        \end{minipage}
        \hfill
        \begin{minipage}{0.245\linewidth}
            \centering
            \includegraphics[width=.99\linewidth]{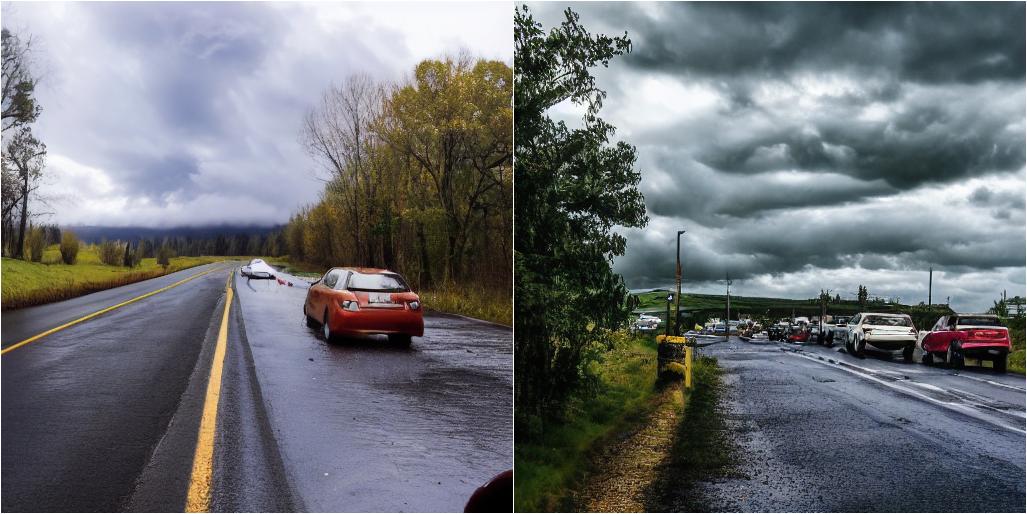}
        \end{minipage}
        \hfill
        \begin{minipage}{0.245\linewidth}
            \centering
            \includegraphics[width=.99\linewidth]{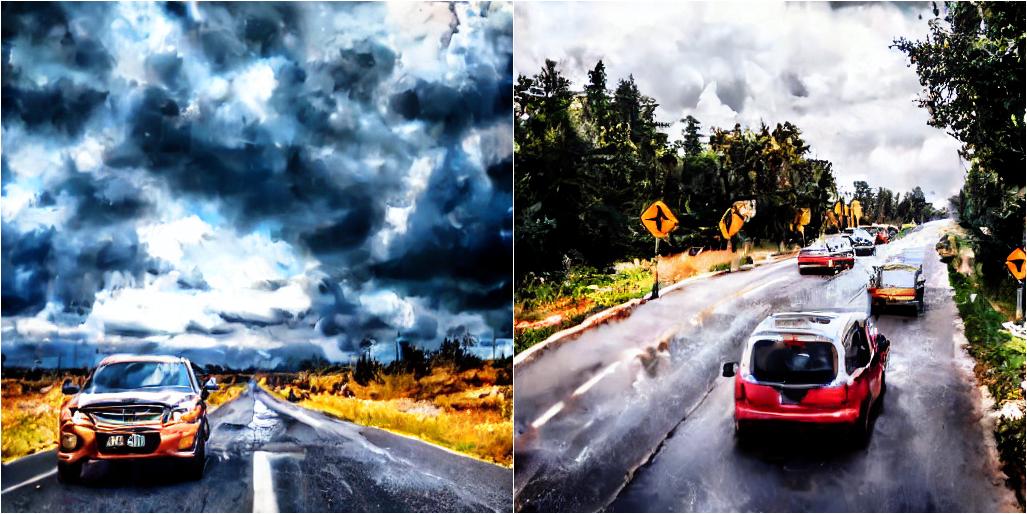}
        \end{minipage}
        \begin{minipage}{1.\linewidth}
            \smallskip
            \centering
            {\footnotesize Several cars drive down the road on a cloudy day.\par}
        \end{minipage}
    \end{minipage}
    \\ \vspace{.5mm}
    \begin{minipage}{0.99\linewidth}
        \centering
        \begin{minipage}{0.245\linewidth}
            \centering
            \includegraphics[width=.99\linewidth]{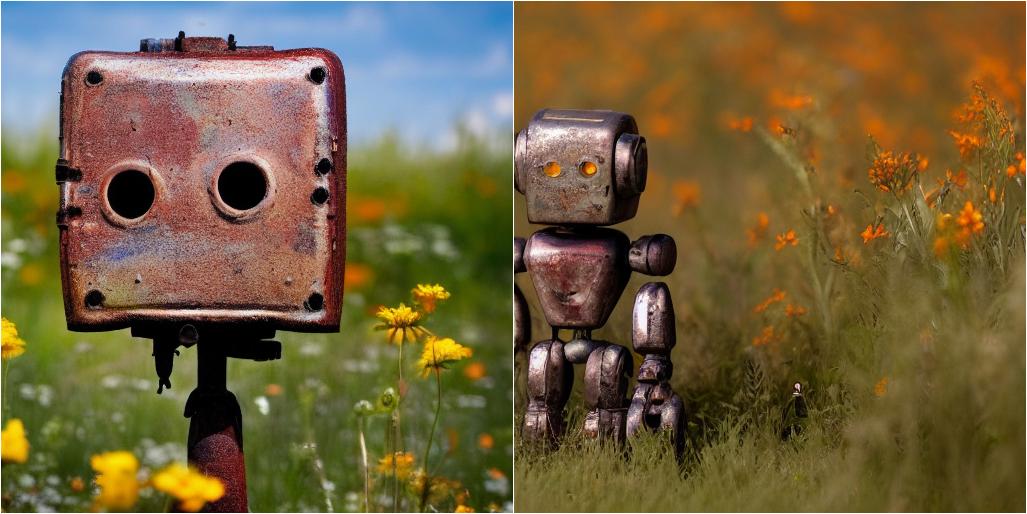}
        \end{minipage}
        \hfill
        \begin{minipage}{0.245\linewidth}
            \centering
            \includegraphics[width=.99\linewidth]{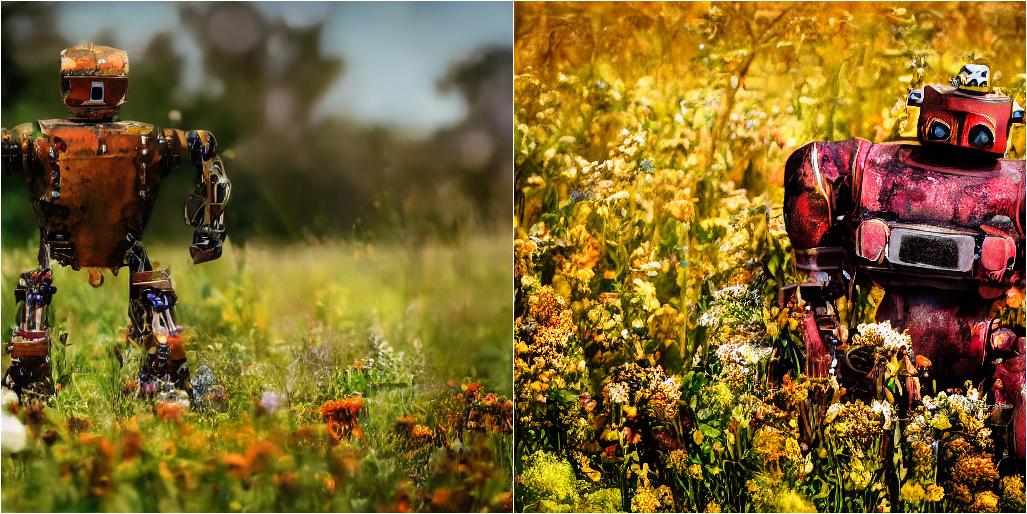}
        \end{minipage}
        \hfill
        \begin{minipage}{0.245\linewidth}
            \centering
            \includegraphics[width=.99\linewidth]{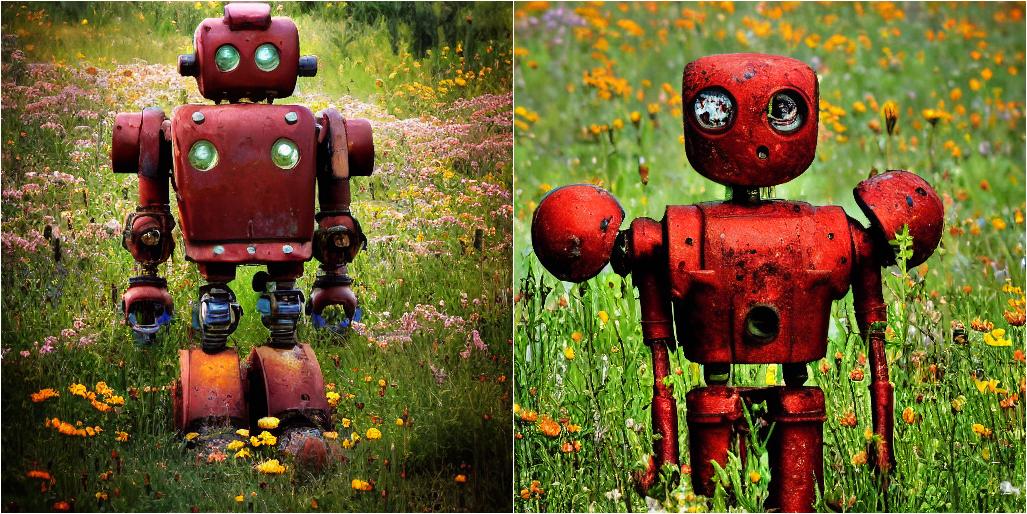}
        \end{minipage}
        \hfill
        \begin{minipage}{0.245\linewidth}
            \centering
            \includegraphics[width=.99\linewidth]{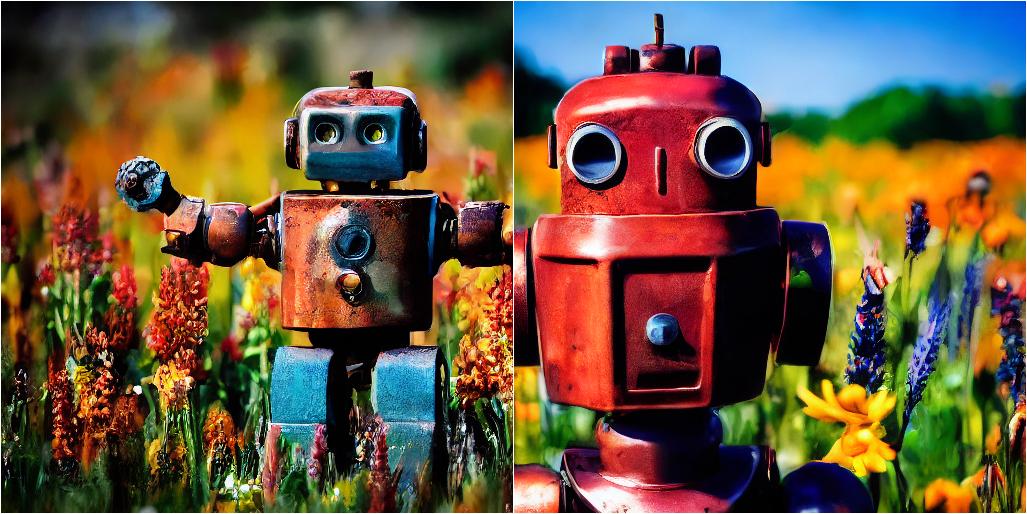}
        \end{minipage}
        \begin{minipage}{1.\linewidth}
            \smallskip
            \centering
            {\footnotesize A rusted robot lying in a field of wildflowers, ultra-detailed, shallow depth of field.\par}
        \end{minipage}
    \end{minipage}
    \\ \vspace{.5mm}
    \begin{minipage}{0.99\linewidth}
        \centering
        \begin{minipage}{0.245\linewidth}
            \centering
            \includegraphics[width=.99\linewidth]{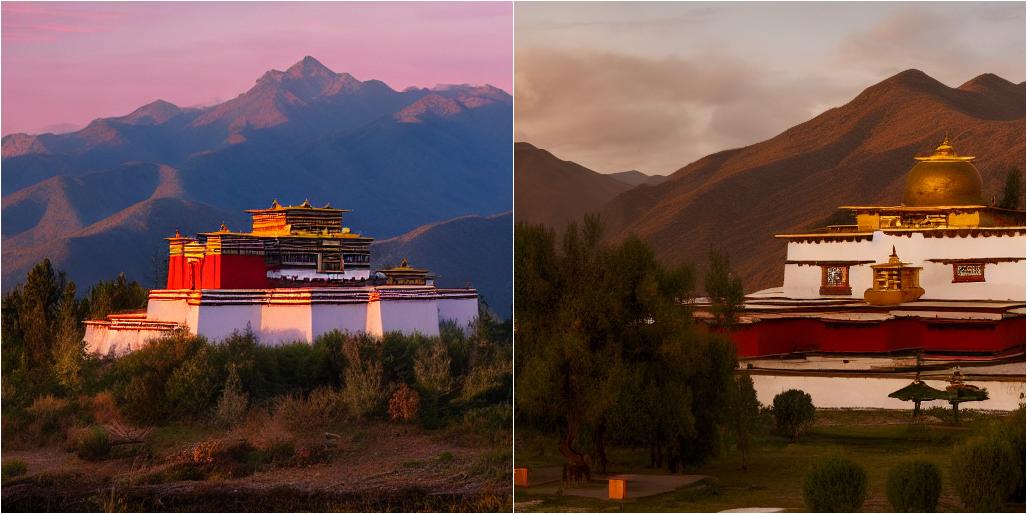}
        \end{minipage}
        \hfill
        \begin{minipage}{0.245\linewidth}
            \centering
            \includegraphics[width=.99\linewidth]{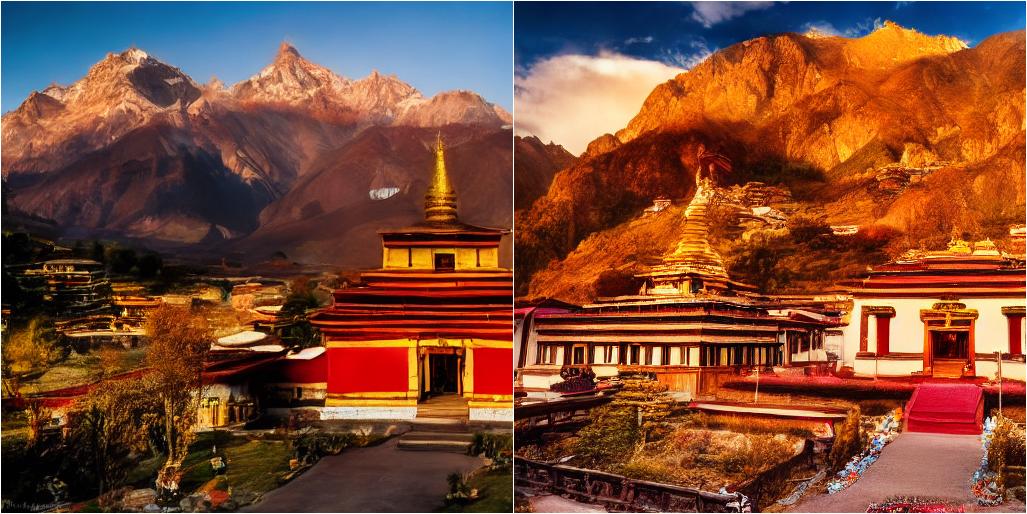}
        \end{minipage}
        \hfill
        \begin{minipage}{0.245\linewidth}
            \centering
            \includegraphics[width=.99\linewidth]{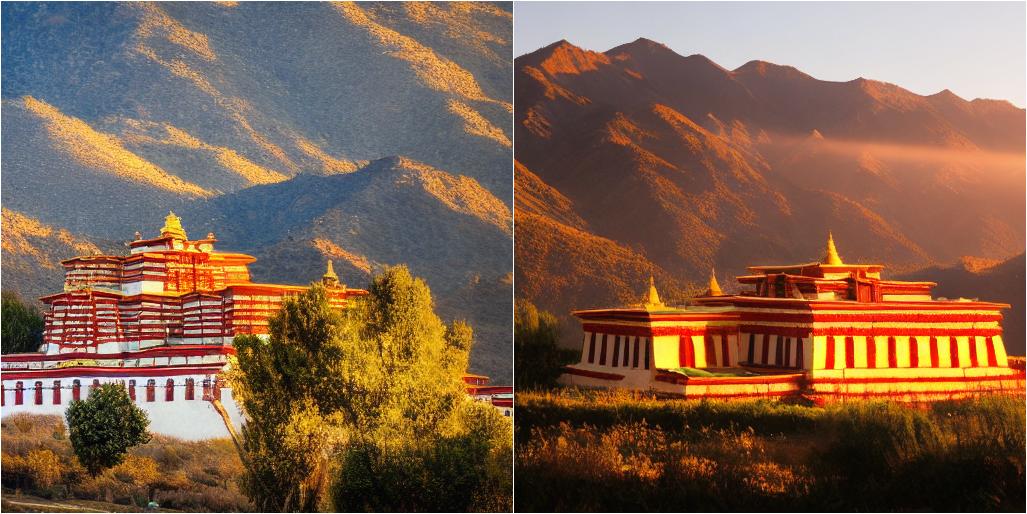}
        \end{minipage}
        \hfill
        \begin{minipage}{0.245\linewidth}
            \centering
            \includegraphics[width=.99\linewidth]{figures/gen_samples/sd1.5/samples/A_quiet_Tibetan_monastery_at_sunrise_mountains_in_the_background_soft_golden_light_vmpor2g.jpeg}
        \end{minipage}
        \begin{minipage}{1.\linewidth}
            \smallskip
            \centering
            {\footnotesize A quiet Tibetan monastery at sunrise, mountains in the background, soft golden light.\par}
        \end{minipage}
    \end{minipage}
    \vspace{-2mm}
    \caption{Illustration of the generated samples of different models.}
    \label{fig:hps-samples-comp}
    \vspace{-4mm}
\end{figure}

\begin{wrapfigure}{r}{0.47\linewidth}
    \centering
    \vspace{-2mm}
    \includegraphics[width=.98\linewidth]{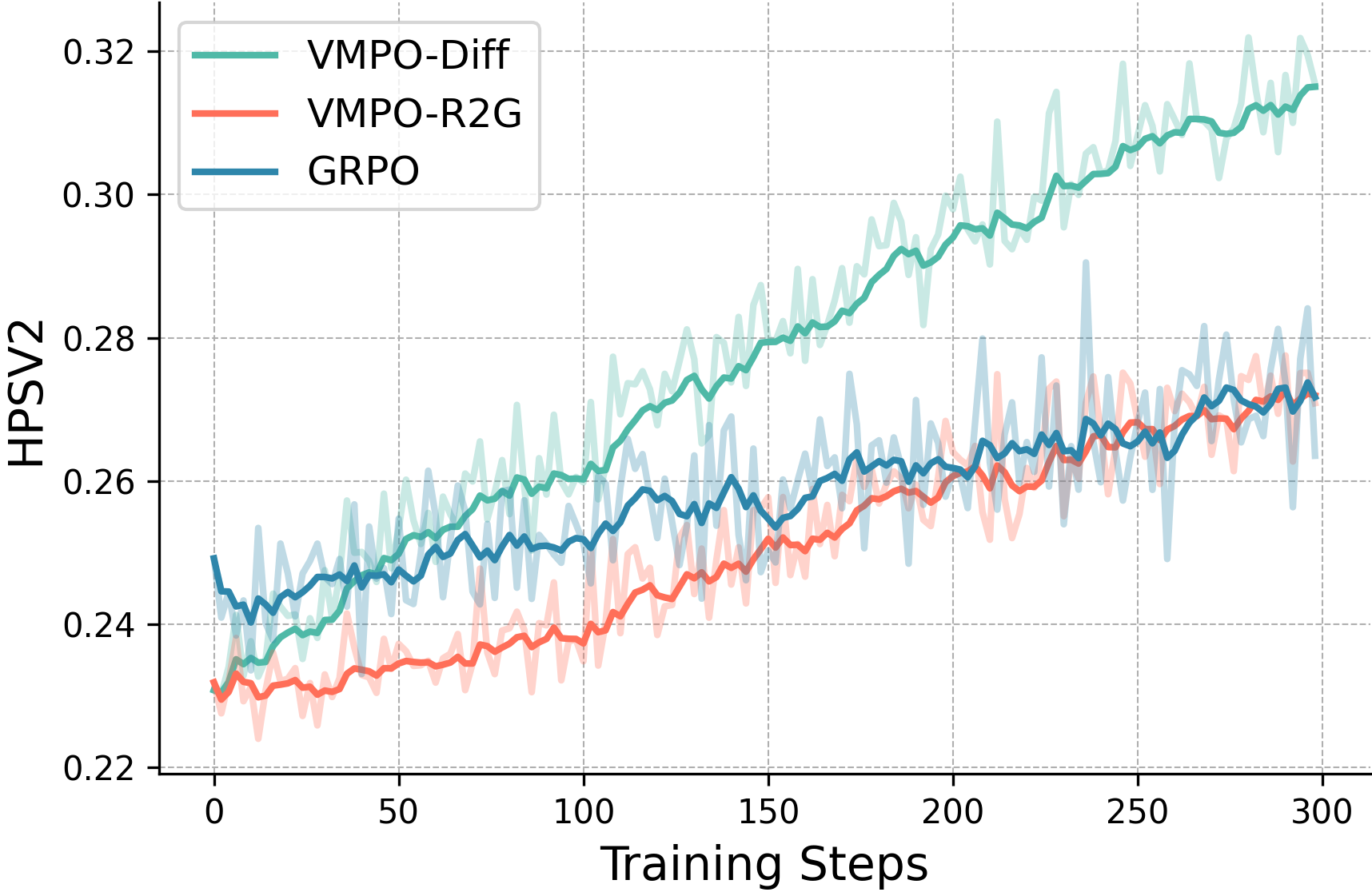}
    \vspace{-2mm}
    \caption{HPSv2 convergence curves of SD1.5.}
    \label{fig:hps_curve}
    \vspace{-3mm}
\end{wrapfigure}
\textbf{Results and Discussion.}
We present the reward convergence curves in \cref{fig:hps_curve}. As shown in the figure, both VMPO and GRPO steadily improve HPSv2 over training, indicating effective reward optimisation. Notably, VMPO-Diff consistently attains higher reward values throughout training, suggesting improved sample efficiency for reward alignment compared to GRPO and VMPO-R2G.
To further illustrate the alignment dynamics, \cref{fig:t2i-hpsv2-training} visualises the evolution of generated images during training, where the images progressively align more faithfully to the prompts, qualitatively validating the effectiveness of VMPO.
In \cref{tab:hps-results}, we quantitatively benchmark VMPO against GRPO. It shows that VMPO achieves higher HPSv2 and ImageReward scores, which are both preference-trained, CLIP-based rewards.
However, we observe a drop in CLIPScore and DreamSim, which respectively measure text–image alignment and diversity, indicating that both VMPO and GRPO are susceptible to reward hacking.
We illustrate this effect with representative samples in \cref{fig:hps-samples-comp} and additional examples provided in \cref{fig:hps-samples-comp-appendix}.

%% file: sections/04-conclusion.tex
\section{Conclusion and Limitation}
In this paper, we proposed variance minimisation policy optimiser (VMPO), a diffusion alignment method that learns the denoising policy by minimising the variance of importance weights. 
VMPO closely resembles proximal policy optimisation when Monte Carlo estimation is employed, while retaining substantial flexibility. By choosing different potential functions along sampling trajectories and alternative variance minimisation objectives, VMPO naturally connects to several existing approaches.  
Empirically, we demonstrated the effectiveness of VMPO by fine-tuning Stable Diffusion v1.5 and v3.5. We further discuss limitations and directions for future work in \cref{sec:appendix-limitation-future-work}.

%% file: sections/07-appendix.tex
\newpage
\appendix

\begin{center}
\LARGE
\textbf{Appendix for ``Diffusion Alignment Beyond KL: Variance Minimisation as Effective Policy Optimiser''}
\end{center}

\etocdepthtag.toc{mtappendix}
\etocsettagdepth{mtchapter}{none}
\etocsettagdepth{mtappendix}{subsection}
{\small \tableofcontents}

\section{Diffusion Alignment: a Tale of Two Views} \label{sec:appendix-align-two-view}

Let $q(x_0)$ be the data distribution, denoising diffusion models \citep{sohl2015deep,ho2020denoising,kingma2021variational} construct a Markovian process $q(x_{0:T}) = q(x_0)\prod_{t=1}^T q(x_t | x_{t-1})$, with $q(x_{t} | x_{t-1}) = \mathcal{N}(\alpha_t / \alpha_{t-1} x_t, (1 - \alpha_t^2/ \alpha_{t-1}^2 )I)$. By setting $\sigma_t^2 + \alpha_t^2 = 1$, this noising process has a simple marginal in the form of $q(x_t | x_0) = \mathcal{N}(x_t; \alpha_t x_0; \sigma_t^2 I)$, where $\{\alpha_t, \sigma_t\}$ are pre-defined signal-noise schedules.
Diffusion models parametrise the denoising process as a Markovian process as well, which takes the form of
\begin{equation}
\begin{gathered} \label{eq:ddpm-model}
    p_\theta(x_0) = \int p_\theta(x_{0:T}) \dif x_{0:T} = \int p_\theta(x_T) \prod_{t=1}^T p_\theta(x_{t-1} | x_t) \dif x_{1:T}, \\
    p_\theta (x_{t-1} | x_t) = \mathcal{N}\left(x_{t-1}; \frac{\alpha_t \sigma_{t-1}^2}{\alpha_{t-1} \sigma_t^2} x_t +  \frac{\alpha_{t-1}^2 - \alpha_t^2}{\alpha_{t-1} \sigma_t^2} \hat{x}_\theta(x_t, t) \right)
\end{gathered}
\end{equation}
in which $\hat{x}_\theta(x_t, t)$ is a neural network to predict the clean data $x_0$ given the noisy input $x_t = \alpha_t x_0 + \sigma_t \epsilon$. Thus, diffusion models can be trained via data prediction loss
\begin{align}
    \mathcal{L} (\theta) = \E_{t, x_0 \sim q(x_0), \epsilon \sim \mathcal{N}(0,I)} \lVert x_0 - \hat{x}_\theta (\alpha_t x_0 + \sigma_t \epsilon, t) \rVert^2.
\end{align}
In condition generation, such as text-to-image generation, one can parametrise the diffusion model as
\begin{align}
     p_\theta (x_0; c) = \int p_\theta(x_T) \prod_{t=1}^T p_\theta(x_{t-1} | x_t; c) \dif x_{1:T},
\end{align}
where $c$ denotes the condition (e.g., text), and replace $\hat{x}_t (\alpha_t x_0 + \sigma_t \epsilon, t)$ with $\hat{x}_t (\alpha_t x_0 + \sigma_t \epsilon, t, c)$. 
In this paper, we study diffusion alignment. Given a diffusion model $p_\theta$ and a reward model $r$, our goal is to maximise the reward of the generated samples.
We present two complementary perspectives for achieving this objective and show that they ultimately lead to the same optimisation problem.
In what follows, we occasionally ignore the dependence on $c$ without loss of generality.

\subsection{Diffusion Alignment as Policy Optimisation}

By the Markov property of denoising diffusion models (cf.\ \cref{eq:ddpm-model}), the denoising process can be cast as a finite-horizon Markov Decision Process (MDP) $(\mathcal{S}, \mathcal{A}, P_0, P, R)$ \citep{black2023training,fan2023dpok}, where
\begin{equation}
\begin{gathered}
    s_t = (x_{T-t}, c), \quad a_t = x_{T-t-1}, \quad P_0 (s_0) = (\mathcal{N}(0,I), p(c)), \quad P(s_{t+1} | s_t, a_t) = (\delta_{a_t}, \delta_{c}) \nonumber \\
    R(s_t, a_t) = r(s_{t+1})\ \text{only if}\ t = T-1\ \text{else}\ 0, \quad \pi_\theta (a_t | s_t) = p_\theta (x_{T-t-1} | x_{T-t}, c).
\end{gathered}
\end{equation}
Here, $s_t$ and $a_t$ denote the state and action at time step $t$; $P_0$ and $P$ are the initial state distribution and the deterministic transition kernel; $R$ is a sparse terminal reward defined only at the final denoising step; and $\pi_\theta$ corresponds to the denoising policy parameterised by $\theta$.
Under this formulation, diffusion alignment amounts to optimising the policy $\pi_\theta$ to maximise the expected terminal reward. This leads to the following diffusion alignment objective \citep{black2023training}
\begin{align}
    \mathcal{J}(\theta) = \E_{\tau \sim \pi_\theta} \left[\sum_{t=0}^T G_t \right]
\end{align}
where $\tau \triangleq (x_0, \dots, x_T)$ denotes the sample trajectory and $G_t$ is the return-to-go, which in the diffusion alignment setting reduces to $G_t = r(x_0)$.
Applying the policy gradient theorem, one can show that \citep[Lemma~4.1]{fan2023dpok}:
\begin{align} \label{eq:ddpo-obj}
    \nabla_\theta \mathcal{J} (\theta) 
    &= \! \E_{\tau \sim \mu_\theta} \left[\sum_{t=0}^T G_t \nabla_{\!\theta} \log \pi_\theta (a_t | x_t)\right] \nonumber \\
    &= \sum_{t=1}^T \E_{p_\theta (x_{t-1} | x_t)} [r(x_{t-1}) \nabla_\theta \log p_\theta (x_{t-1} | x_t)] \nonumber \\
    &= \nabla_\theta \E_{p_\theta (x_0)} [r(x_0)]
\end{align}

This expression shows that diffusion alignment can be viewed as reinforcing denoising transitions that lead to high-reward final samples.

\subsection{Diffusion Alignment as Probability Inference} \label{sec:diffusion-alignment-prob-inf}
Alternatively, diffusion alignment can be formulated from a probabilistic inference perspective. Given a pretrained diffusion model $\pref (x_{t-1} \mid x_t)$ and a reward model $r(x)$, one can define a reward-tilted target distribution $\ptilt (x_{t-1} | x_t) \propto \pref (x_{t-1}|x_t) \exp(r(x_{t-1}) / \beta)$ where $\beta > 0$ controls the strength of reward guidance.
A natural approach is to learn a model $p_\theta$ that approximates $\ptilt$ by minimising the KL divergence $\mathbb{KL}(p_\theta \| \ptilt)$. Taking the gradient with respect to $\theta$, we obtain
\begin{align}
    &\nabla_\theta \mathbb{KL}\left(p_\theta (x_{t-1} | x_t) || \ptilt (x_{t-1} | x_t) \right) \nonumber \\
    &= \int \left( \log \frac{p_\theta (x_{t-1}|x_t)}{\pref (x_{t-1} | x_t) }  - \frac{1}{\beta} r(x_{t-1}) \right) \nabla_\theta p_\theta (x_{t-1}|x_t) \dif x_{t-1} \nonumber \\
        &\qquad + \canceltozero{\int \log Z \nabla_\theta p_\theta (x_{t-1}|x_t) \dif x_{t-1}} + \canceltozero{\int  p_\theta (x_{t-1}|x_t) \nabla_\theta \log p_\theta  (x_{t-1}|x_t) \dif x_{t-1}} \nonumber \\
    &= \nabla_\theta \E_{p_\theta (x_{t-1} | x_t)} \left[- \frac{1}{\beta} r(x_{t-1})\right] + \nabla_\theta \mathbb{KL}(p_\theta (x_{t-1} | x_t) || \pref(x_{t-1} | x_t)).
\end{align}
where we have used the fact that the normalisation constant of $\ptilt$ does not depend on $\theta$.
This leads to the following alignment objective:
\begin{align} 
    \argmax_\theta \sum_{t=1}^T \E_{p_\theta (x_{t-1} | x_t)} \left[r(x_{t-1})\right] - \beta \mathbb{KL}(p_\theta (x_{t-1} | x_t) || \pref(x_{t-1} | x_t))
\end{align}
We observe that this objective closely resembles the one in \cref{eq:ddpo-obj}, but includes an additional KL regularisation term that constrains the learned policy to remain close to the pretrained diffusion model. Such KL regularisation is standard in reinforcement learning \citep{schulman2015trust,schulman2017proximal,haarnoja2018soft,levine2018reinforcement} and variational inference \citep{kingma2013auto,blei2017variational}, where it serves to stabilise optimisation, prevent mode collapse, and preserve the prior.

\subsection{REINFORCE, PPO, and GRPO}
REINFORCE \citep{williams1992simple} is a Monte Carlo policy-gradient estimator that expresses the gradient of the expected return as an expectation of the reward weighted by the score function:
\begin{align}
    \nabla_\theta \E_{\pi_\theta} [G_t] = \E_{\pi_\theta} [G_t \nabla_\theta \log \pi_\theta (a_t | s_t)].
\end{align}
Although unbiased, the REINFORCE estimator typically suffers from high variance.
A standard variance-reduction technique is to introduce a control variate \citep{ranganath2014black} in the form of a baseline $b(s_t)$, which is subtracted from the return-to-go $G_t$.
In reinforcement learning, a common choice is the value function $V_{\pi_\theta}(s_t)$ \citep{schulman2015high}, which is action-independent and therefore preserves unbiasedness, since $\E_{\pi_\theta} [b(s_t) \nabla_\theta \log \pi_\theta (a_t | s_t)] = 0$. This yields the following gradient estimator:
\begin{align}
    \!\!\nabla_\theta \mathcal{J} (\theta) 
    \!=\! \E_{\tau \!\sim\! \pi_\theta} \!\!\left[
        \sum_{t=0}^T (G_t \!-\! b(s_t)) \nabla_\theta \log \pi_\theta (a_t | s_t)
    \right]
    \!=\! \E_{\tau \!\sim\! \pi_\theta} \!\!\left[
        \sum_{t=0}^T \hat{A}_t \nabla_\theta \log \pi_\theta (a_t | s_t)
    \right],
\end{align}
where $\hat{A}_t = G_t - b(s_t)$ is an estimate of the advance $Q_{\pi_\theta}(s_t, a_t) - V_{\pi_\theta}(s_t)$, where $Q$ denotes the Q-value. This leads to the REINFORCE surrogate objective:
\begin{align}
    \mathcal{J}_{\mathrm{REINFORCE}} (\theta) = \E_{\tau \sim \pi_{\theta_{\mathrm{sg}}}} \left[ \sum_{t=0}^T \hat{A}_t \log \pi_\theta (a_t | s_t) \right].
\end{align}
However, $\mathcal{J}_{\mathrm{REINFORCE}}$ is inefficient in practice, since it requires sample trajectries for each policy update. This can also lead to excessively large policy updates, resulting in instability and degraded performance.
Proximal Policy Optimisation (PPO) \citep{schulman2017proximal} addresses these issues by explicitly constraining the deviation between the updated policy $\pi_\theta$ and the reference policy $\pi_{\theta_{\mathrm{sg}}}$, by applying importance sampling with the weight $\rho_t \triangleq \frac{\pi_\theta(a_t | s_t)}{\pi_{\theta_{\mathrm{sg}}}(a_t | s_t)}$, and clipping it to remain within a trust region. The PPO clipped surrogate objective is given by
\begin{align}
    \mathcal{J}_{\mathrm{PPO}}(\theta) =
    \E_{\tau \sim \pi_{\theta_{\mathrm{sg}}}} \left[
        \sum_{t=0}^T
        \min \left(
            \rho_t(\theta) \hat{A}_t,
            \mathrm{clip}\left(\rho_t(\theta), 1-\epsilon, 1+\epsilon\right)\hat{A}_t
        \right)
    \right],
\end{align}
where $\epsilon > 0$ controls the size of the trust region.
Intuitively, PPO preserves the variance-reduction benefits of advantage-weighted policy gradients while preventing overly aggressive updates that would invalidate the on-policy approximation underlying REINFORCE.

While PPO stabilises optimisation by constraining policy updates, it relies on an explit value function $V_{\pi_\theta}$ to construct the advantage estimates. In many settings, particularly those with sparse, trajectory-level rewards, learning an accurate reward value function can be difficult and may introduce addtional instability.
Group Relative Policy Optimisation (GRPO) \citep{shao2024deepseekmath} circumvents this issue by eliminating the value function altogether and instead using relative advantages computed within groups of samples.
Concretely, for a given state $s_t$, GRPO samples a group of actions $\{a_t^i\}_{i=1}^K \sim \pi_{\theta_{\mathrm{sg}}} (\cdot | s_t)$ and computes their corresponding returns $\{G_t^i\}_{i=1}^K$. 
The advantages for each action is then defined relative to the group mean:
\begin{align} \label{eq:advantage-grpo}
    \hat{A}_t^i = G_t^i - \frac{1}{K} \sum_{j=1}^K G_t^i.
\end{align}
This group-normalised advantage can be written as $\hat{A}_t^i = \frac{K-1}{K} \left( G_t^i - \frac{1}{K} \sum_{j \neq i} G_t^j \right)$, in which the second term is a leave-one-out (LOO) baseline \citep{salimans2014using,kool2019buy}. This baseline acts as a control variate, reducing variance without requiring a learned value function.\footnote{We intentionally do not normalise the advantage by the group standard deviation, following Dr. GRPO \citep{liu2025understanding}. Mean-only normalisation yields a REINFORCE estimator with a leave-one-out control variate whose gradient is an unbiased estimator of $\nabla_\theta \mathcal{J}(\theta)$. In contrast, the original GRPO formulation additionally normalises by the group standard deviation, which introduces bias into the policy gradient estimator.}
GRPO combines this relative advantage estimator with a PPO-style clipped objective, yielding the objective:
\begin{align}
    \mathcal{J}_{\mathrm{GRPO}}(\theta) =
    \E_{\{\tau_i\} \sim \pi_{\theta_{\mathrm{sg}}}} \left[
        \frac{1}{K} \sum_{i=1}^K
        \sum_{t=0}^T
        \min \left(
            \rho_t(\theta)\hat{A}_t^i,
            \mathrm{clip} \left(\rho_t(\theta), 1-\epsilon, 1+\epsilon\right)\hat{A}_t^i
        \right)
    \right].
\end{align}
From this viewpoint, GRPO inherits the stability guarantees of PPO through proximal updates, while avoiding the need to learn a parametric value function. This makes GRPO particularly appealing in settings where rewards are defined at the trajectory level or where value estimation is challenging, as is the case in diffusion alignment.

In practice, policy optimisation is often further regularised by penalising deviations from a reference policy $\pi_{\mathrm{ref}}$ via a KL divergence term, leading to the following objective:
\begin{align} \label{eq:final-rl-obj}
    \argmax_\theta \mathcal{J}_{\mathrm{REINFORCE}/\mathrm{PPO}/\mathrm{GRPO}} (\theta) - \beta \sum_{t=0}^T \mathbb{KL}(\pi_\theta (a_t | s_t) || \pi_{\mathrm{ref}} (a_t | s_t)).
\end{align}
Such KL regularisation plays a central role in stabilising training by constraining the updated policy to remain close to a fixed reference.
In many practical implementations, particularly in reinforcement learning from human feedback (RLHF) \citep{ouyang2022training} for large language models, this regularisation is applied implicitly rather than as an explicit penalty in the optimisation objective \citep{zhang2025design,shah2025comedy}. Specifically, the KL term is absorbed into the reward before advantage estimation by modifying the return as
\begin{align} \label{eq:kl-in-reward}
    G_t^\prime = G_t - \beta \log \frac{\pi_\theta (a_t | s_t)}{\pi_{\mathrm{ref}}(a_t | s_t)}.
\end{align}
The resulting modified return is then used to compute the advantage $\hat{A}_t$, effectively enforcing KL regularisation at the level of the reward signal. This implicit formulation contrasts with explicitly adding a KL penalty to the final objective, as in \cref{eq:final-rl-obj}, while yielding similar regularising effects in practice.

\section{Proofs and Derivations}

\subsection{Proof of \cref{prop:optimal-proposal}}

\restapropone*
\begin{proof}
    We first prove that: $p_{\theta^*} = \ptilt(x_{t-1} | x_t)$, $\theta^* = \argmin_\theta \mathcal{L}^h_{\mathrm{Var}}(t;\theta)$.
    This can be straightforwardly validated by observing that the minimum variance implies the importance weight is a constant for all $(x_{t-1}, x_t)$. Therefore, we have 
    \begin{align}
        \frac{p_{\mathrm{ref}}(x_{t-1} | x_t) \exp(\frac{r(x_{t-1})}{\beta})}{p_{\theta^*} (x_{t-1} | x_t)} = const. \quad \forall x_{t-1}, x_t.
    \end{align}
    Therefore, $p_{\theta^*} (x_{t-1} | x_t) = \ptilt (x_{t-1} | x_t) \propto p_{\mathrm{ref}}(x_{t-1} | x_t) \exp(\frac{r(x_{t-1})}{\beta})$.
    We then proceed to prove $\left. \nabla_\theta \mathcal{L}^h_{\mathrm{Var}}(t;\theta) \right|_{h=p_\theta} = \nabla_\theta \mathbb{KL}(p_\theta (x_{t-1} | x_t) || \ptilt(x_{t-1} | x_t))$.
    To see this, we define 
    \begin{align}
        w_\theta (x_{t-1}, x_t) = \frac{p_{\mathrm{ref}}(x_{t-1} | x_t) \exp(\frac{r(x_{t-1})}{\beta})}{p_{\theta} (x_{t-1} | x_t)} \quad \Rightarrow \quad \nabla_\theta \log w_\theta = - \nabla_\theta \log p_\theta (x_{t-1} | x_t)
    \end{align}
    Using the definition from \cref{eq:log-var-loss}, we see that 
    \begin{align}
        \nabla_\theta \mathcal{L}^h_{\mathrm{Var}}(t;\theta)
        &= \frac{1}{2} \nabla_\theta \int h \log^2 w_\theta \dif x_{t-1} - \frac{1}{2} \nabla_\theta \left( \int h \log w_\theta  \dif x_{t-1} \right)^2 \nonumber \\
        &= \int h \log w_\theta \nabla_\theta  \log w_\theta  \dif x_{t-1} - \left( \int h \log w_\theta  \dif x_{t-1} \right) \left( \int h \nabla_\theta  \log w_\theta  \dif x_{t-1} \right). \nonumber
    \end{align}
    By setting $h = p_\theta$ and noting that $\nabla_\theta \log w_\theta = - \nabla_\theta \log p_\theta (x_{t-1} | x_t)$, we have
    \begin{align}
        \left. \nabla_\theta \mathcal{L}^h_{\mathrm{Var}}(t;\theta) \right|_{h=p_\theta}
        &= \int - \log w_\theta \nabla_\theta  p_\theta   \dif x_{t-1} - \left( \int p_\theta \log w_\theta  \dif x_{t-1} \right) \canceltozero{\left( \int - p_\theta \nabla_\theta  \log p_\theta  \dif x_{t-1} \right)} \nonumber \\
        &= \nabla_\theta \E_{p_\theta (x_{t-1} | x_t)} \left[ -\frac{1}{\beta} r(x_{t-1}) \right] + \int \log \frac{p_\theta (x_{t-1} | x_t)}{\pref (x_{t-1} | x_t)} \nabla_\theta p_\theta (x_{t-1} | x_t) \dif x_{t-1} \nonumber \\
        &= \nabla_\theta \E_{p_\theta (x_{t-1} | x_t)} \left[ -\frac{1}{\beta} r(x_{t-1}) \right] + \nabla_\theta \mathbb{KL}(p_\theta (x_{t-1} | x_t) || \pref (x_{t-1} | x_t)) \nonumber \\
        &= \nabla_\theta \mathbb{KL}(p_\theta (x_{t-1} | x_t) || \ptilt(x_{t-1} | x_t)), \nonumber
    \end{align}
    which completes the proof.
\end{proof}

\subsection{Proof of \cref{eq:graidnet-monte-carlo-gradient}} \label{sec:proof-unbias-gradient}
In this section, we show that although the Monte Carlo estimator in \cref{eq:log-var-monte-carlo} is biased, its gradient remains an unbiased estimator when the reference distribution is chosen as $h = p_\theta$. We establish this by explicitly analysing the gradient of the estimator:
\begin{align} \label{eq:appendix-grad-log-var}
    \nabla_\theta \hat{\mathcal{L}}^h_{\mathrm{Var}}(t;\theta)
    &= \frac{1}{(K-1)} \sum_{i=1}^K A_t^i (\theta) \left( \nabla_\theta \log R_\theta (x^i_{t-1} | x_t^i) - \nabla_\theta \overline{\log R_\theta (x_{t-1} | x_t)}  \right) \nonumber \\
    &= \frac{1}{(K-1)} \left( \sum_{i=1}^K A_t^i (\theta)  \nabla_\theta \log R_\theta (x^i_{t-1} | x_t^i) - \nabla_\theta \overline{\log R_\theta (x_{t-1} | x_t)} \canceltozero{\sum_{i=1}^K A_t^i}  \right) \nonumber \\
    &= \frac{1}{(K-1)} \sum_{i=1}^K - A_t^i (\theta)  \nabla_\theta \log p_\theta (x_{t-1} | x_t).
\end{align}
To simplify notation, define 
\begin{align}
    f_t^i \triangleq \log R_\theta (x^i_{t-1} | x_t^i) \!+\! \frac{1}{\beta} r(x^i_{t-1}), \quad \Bar{f}_t = \frac{1}{K} \sum_{j=1}^K f_t^j.
\end{align}
Then we can rewrite
\begin{align} \label{eq:grad-log-var-loo}
    \nabla_\theta \hat{\mathcal{L}}^h_{\mathrm{Var}}(t;\theta)
    &= \frac{-1}{K-1} \sum_{i=1}^K \left( f_t^i - \Bar{f}_t \right) \nabla_\theta \log p_\theta (x_{t-1}^i | x_t^i)  \\
    &= - \sum_{i=1}^K \left( \frac{1}{K} f_t^i - \frac{1}{K(K-1)} \sum_{j \neq i} f_t^j \right) \nabla_\theta \log p_\theta (x_{t-1}^i | x_t^i).
\end{align}
Now setting $h = p_\theta$, the second term inside the parentheses vanishes in expectation
\begin{align}
    \E_{p_\theta (x_{t-1} | x_t)} \left [\nabla_\theta \log p_\theta (x_{t-1}^i | x_t^i) \sum_{j \neq i} f_t^j  \right] = \sum_{j \neq i} f_t^j \nabla_\theta  \int p_\theta (x_{t-1}^i | x_t^i) \dif x_{t-1}^i = 0.
\end{align}
since the score function integrates to zero. As a result, taking expectations yields
\begin{align}
    \!\!\!\E_{p_\theta} \left[\left. \nabla_\theta \hat{\mathcal{L}}^h_{\mathrm{Var}}(t;\theta) \right|_{h=p_\theta} \! \right] 
    &= - \frac{1}{K} \sum_{i=1}^K \E_{p_\theta} \left[ f_t^i \nabla_\theta \log p_\theta (x_{t-1}^i | x_t^i) \right] \nonumber \\
    &= - \E_{p_\theta} \left[
       \left( \log \frac{p_{\mathrm{ref}}(x_{t-1} | x_t))}{p_\theta (x_{t-1} | x_t)} + \frac{1}{\beta} r(x_{t-1}) \right) \nabla_\theta \log p_\theta (x_{t-1} | x_t)
    \right] \nonumber \\ 
    &= \nabla_\theta \E_{p_\theta (x_{t-1} | x_t)}\left[ -\frac{1}{\beta} r(x_{t-1}) \right] + \nabla_\theta \mathbb{KL}(p_\theta (x_{t-1} | x_t) || \pref (x_{t-1} | x_t)) \nonumber \\
        &= \nabla_\theta \mathbb{KL}(p_\theta (x_{t-1} | x_t) || \ptilt(x_{t-1} | x_t)). \nonumber
\end{align}
Hence, although $\hat{\mathcal{L}}^h_{\mathrm{Var}}(t;\theta)$ is a biased estimator of the variance objective, its gradient is an unbiased estimator of the true gradient when $h = p_\theta$.

\textbf{Remark.} The term $\Bar{f}_t$ in \cref{eq:grad-log-var-loo} acts as a leave-one-out baseline, reducing variance without introducing bias. Motivated by \cref{eq:appendix-grad-log-var}, we may instead define the following surrogate loss:
\begin{align}
    \hat{\mathcal{L}}_{h=p_\theta}^{\mathrm{Var}}(\theta) 
    = \frac{1}{(K-1)} \sum_{i=1}^K - A_t^i (\theta_{\mathrm{sg}})  \log p_\theta (x_{t-1} | x_t).
\end{align}
Applying importance sampling and PPO-style clipping yields the following objective
\begin{align} \label{eq:appendix-vamp3-obj}
    \argmax_\theta
    \frac{1}{K-1} \sum_{i=1}^K
        \sum_{t=0}^T
        \min \left(
            \rho^i_t(\theta){A}_t^i,
            \mathrm{clip} \left(\rho^i_t(\theta), 1-\epsilon, 1+\epsilon\right){A}_t^i
        \right)
\end{align}
with $\rho^i_t(\theta) = \frac{p_\theta (x^i_{t-1} | x^i_t)}{p_{\theta_{\mathrm{old}}} (x^i_{t-1} | x^i_t)}$, $\{x_t^i\} \sim p_{\theta_{\mathrm{old}}}$, and importantly
\begin{align}
    {A}_t^i = \log R_\theta (x^i_{t-1} | x_t^i) \!+\! \frac{1}{\beta} r(x^i_{t-1}) \!-\! \overline{\log R_\theta (x_{t-1} | x_t)} - \frac{1}{\beta} \overline{r(x_{t-1})}.
\end{align}
Compared to the advantage used in GRPO (see \cref{eq:advantage-grpo}), this advantage contains two additional components, $\log R_\theta (x^i_{t-1} | x_t^i)$ and $\overline{\log R_\theta (x_{t-1} | x_t)}$.
The latter again acts as a LOO baseline, while the former corresponds to a KL-regularisation term subtracted from the reward, as in \cref{eq:kl-in-reward}.

\section{Holistic VMPO Kaleidoscopes} \label{sec:appendix-holistic-vmpo}
In \cref{sec:vmpo}, we derive the VMPO objective by restricting attention to consecutive state pairs $(x_{t-1}, x_t)$. While \cref{prop:optimal-proposal} establishes the validity of this formulation, it characterises only a ``local'' optimisation perspective.
In this section, we reinterpret the VMPO objective through the lens of sequential Monte Carlo \citep{ou2025inference,ou2025discrete,he2025rne}. From this viewpoint, we show that the variance of the log importance weight of the full trajectory is upper-bounded by the sum of per-timestep variances,
\begin{equation} \label{eq:log-var-inequality}
\begin{gathered}
    \mathbb{V}_h (\log w_{0:T}) \leq T \sum_t \mathbb{V}_h (\log w_t) \\
    w_t = \frac{p_{\mathrm{ref}}(x_{t-1} | x_t)\exp\left(\frac{1}{\beta} r(x_{t-1})\right)}{p_\theta (x_{t-1} | x_t)}; \quad
    w_{0:T} = \frac{p_{\mathrm{ref}}(x_{0:T})\exp\left(\frac{1}{\beta} \sum_t r(x_{t})\right)}{p_\theta (x_{0:T})}.
\end{gathered}
\end{equation}
This result implies that minimising $\sum_t \mathbb{V}_h (\log w_t)$ effectively controls the variance of the importance weights associated with the joint trajectory distribution. In turn, this provides a principled justification for the VMPO objective from a ``global'', trajectory-level perspective.
We proceed by first reviewing the SMC framework and establishing the above inequality. Building on this interpretation, we then propose several VMPO variants motivated by the SMC perspective.

\subsection{Demystifying VMPO} \label{sec:appendix-demystify-vmpo}

In Sequential Monte Carlo \citep{del2006sequential}, we construct a sequence of target distributions $\ptilt (x_{t:T})$ by tilting the base distribution $\pref (x_{t:T})$ with a collection of potential functions $U(x_{t:T})$. Starting from the terminal distribution $\ptilt (x_{T}) \propto \pref (x_T) U(x_T)$, the intermediate targets are defined recursively as
\begin{align}
    \ptilt (x_{t:T}) \propto \pref (x_{t:T}) \prod_{s=t}^T U(x_{s:T}) = \pref (x_t | x_{t+1}) U(x_{t:T}) \ptilt (x_{t+1:T}).
\end{align}
The second equality follows from the Markov property of the pretrained diffusion model $\pref$, which implies $\pref (x_t | x_{t+1:T}) = \pref (x_t | x_{t+1})$. The potentials are required to satisfy the constraint
\begin{align} \label{eq:potential-marginal-constraint}
    \prod_{t=0}^T U(x_{t:T}) = \exp \left( \frac{1}{\beta} r(x_0)\right),
\end{align}
which ensures that the marginal distribution over the final sample obeys $\ptilt (x_{0}) \propto \pref (x_{0}) \exp (r(x_0) / \beta))$. Notably, this constraint can be satisfied by multiple choices of potential functions. We refer readers to \cite{singhal2025general} for a detailed discussion of different constructions.

We now consider estimating the expectation of a test function $\delta$ under the marginal $\ptilt (x_t)$, i.e., $\E_{\ptilt (x_t)}[\delta (x_t)]$. 
When $\delta (\cdot)$ is chosen as the Dirac delta function, this estimation reduces to constructing an empirical approximation of the marginal distribution $\ptilt (x_t)$ itself.
To estimate $\E_{\ptilt (x_t)}[\delta (x_t)]$, we employ importance sampling with proposal distribution $p_\theta$. This yields
\begin{align}
    \E_{\ptilt (x_t)}[\delta (x_t)] \!=\! \E_{p_\theta (x_{t:T})} \!\!\left[\! \frac{\ptilt (x_{t:T})}{p_\theta (x_{t:T})} \delta(x_t) \right] \!\!\approx\!\! \sum_{i=1}^K w_{t:T}^i \delta(x_t), \text{where}\ w_{t:T}^i \!=\! \frac{\ptilt (x_{t:T}^i)}{p_\theta (x_{t:T}^i)}, x_{t:T}^i \!\sim\! p_\theta. \nonumber
\end{align}
To learn an optimal proposal distribution, we minimise the log-variance of the importance weights, following the approach in Section~\ref{sec:vmpo}. Using the identity $\mathbb{V}_h [x] = \E_h [(x - \E_h[x])^2]$ and defining $\overline{\log w_{t:T}} = \E_h [\log w_{t:T}]$, we obtain
\begin{align} \label{eq:log-var-upper-bound}
    \mathbb{V}_h(\log w_{0:T})
    &= \E_{h} \left[ \left| \sum_{t} \log  \frac{\pref (x_{t} | x_{t+1}) U (x_{t:T})}{p_\theta  (x_{t-1} | x_t)} - \overline{\log w_{t:T}} \right|^2 \right] \nonumber \\
    &= T^2 \E_{h} \left[ \left| \sum_{t} \frac{1}{T} \log  \frac{\pref (x_{t} | x_{t+1}) U (x_{t:T})}{p_\theta  (x_{t-1} | x_t)} - \frac{1}{T} \overline{\log w_{t:T}} \right|^2 \right] \nonumber \\
    &\leq T^2 \E_{h} \left[  \sum_{t} \frac{1}{T} \left| \log \frac{\pref (x_{t} | x_{t+1}) U (x_{t:T})}{p_\theta  (x_{t-1} | x_t)} - \overline{\log w_{t:T}} \right|^2  \right] \nonumber \nonumber \\
    &= T^2 \E_{h, t} \left[  \left| \log \frac{\pref (x_{t} | x_{t+1}) U (x_{t:T})}{p_\theta  (x_{t-1} | x_t)} - \overline{\log w_{t:T}} \right|^2  \right] \nonumber \\
    &= T \sum_t \mathbb{V}_h (\log w_{t:T}), 
\end{align}
which establishes the bound $\mathbb{V}_{\!h} \!(\log w_{0:T}) \!\leq\! T \!\sum_t \!\! \mathbb{V}_{\!h} \! (\log w_{t:T})$. By choosing $U\!(x_{t:T}) \!=\! \exp(r(x_t)\!/\!\beta)$, we recover the result in \cref{eq:log-var-inequality}.

This inequality shows that minimising the variance of local importance weights effectively controls the variance of the global importance weight associated with the full trajectory. This property is particularly attractive, as it avoids backpropagating gradients through the entire trajectory; instead, optimisation can be performed using local, per-timestep objectives, leading to more stable training and improved scalability to long horizons.

\subsection{VMPO Kaleidoscopes} \label{sec:appendix-vmpo-kaleidoscopes}
In this section, we propose different design choices for the potential functions $U(x_{t:T})$ used in VMPO.
While all valid potentials must satisfy the global constraint that their product recovers the desired exponential tilting of the terminal reward, this constraint alone does not uniquely specify $U$.
We present several representative choices that offer complementary perspectives on how reward information can be propagated across timesteps.

\textbf{Option 1: Return-to-Go.}
A natural choice of potential is the return-to-go, $U(x_{t:T}) = \exp \left( \frac{1}{T\beta} r(x_0) \right)$, which satisfies $\prod_t U(x_{t:T}) = \exp \left( \frac{1}{\beta} r(x_0) \right)$.
Under this choice, the importance weight for a full trajectory simplifies to
\begin{align}
    w_{0:T} = \frac{\pref (x_{0:T}) \exp \left( \frac{1}{\beta} x_0 \right)}{p_\theta (x_{0:T})} = \frac{\prod_t \pref (x_{t-1} | x_t) \exp \left( \frac{1}{\beta} x_0 \right)}{\prod_t p_\theta (x_{t-1} | x_t)}.
\end{align}
Applying the variance bound derived in \cref{eq:log-var-upper-bound}, we obtain
\begin{align} \label{eq:appendix-log-var-upperbound-r2g}
    \mathbb{V}_h \left( \log \frac{\pref (x_{0:T}) \exp \left( \frac{1}{\beta} x_0 \right)}{p_\theta (x_{0:T})} \right)
    \leq T \sum_t \mathbb{V}_h \left( \log \frac{\pref (x_{t-1} | x_t) \exp \left( \frac{1}{T \beta} r(x_0) \right)}{p_\theta (x_{t-1} | x_t)} \right),
\end{align}
where minimising the right-hand side is equivalent to minimising the objective $\sum_t \mathcal{L}_{h}^{\mathrm{Var}} (t;\theta)$ defined in \cref{eq:log-var-loss}.\footnote{In \cref{eq:log-var-loss}, the factor $\frac{1}{T}$ is omitted from the reward term, as it can be absorbed into the temperature parameter $\beta$.}

\textbf{Option 2: Difference.}
Another widely used choice of potential in SMC is the difference-based construction: $U (x_{t:T}) = \exp \left( \frac{1}{\beta} (r(x_t) - r(x_{t+1})) \right)$ and $U (x_T) = 1$, which satisfies the telescoping constraint $\prod_t U(x_{t:T}) = \exp \left( \frac{1}{\beta} r(x_0) \right)$.
Under this choice, the importance weight of a full trajectory takes the form
\begin{align}
    w_{0:T} = \prod_t \frac{\exp (r(x_{t-1}) / \beta)}{\exp (r(x_{t}) / \beta)} \frac{\pref (x_{t-1} | x_t)}{p_\theta (x_{t-1} | x_t)}.
\end{align}
Applying the variance bound in \cref{eq:log-var-upper-bound}, we obtain
\begin{align}
    \!\!\mathbb{V}_{\!h}\! \left( \! \sum_t \log \frac{\exp (r(x_{t-1}) / \beta)}{\exp (r(x_{t}) / \beta)} \frac{\pref (x_{t-1} | x_t)}{p_\theta (x_{t-1} | x_t)}\! \right)
    \!\leq\!
    T\! \sum_t \! \mathbb{V}_{\!h} \! \left(\! \log \frac{\exp (r(x_{t-1}) / \beta)}{\exp (r(x_{t}) / \beta)} \frac{\pref (x_{t-1} | x_t)}{p_\theta (x_{t-1} | x_t)}\! \right)\!.\!\!
\end{align}
Minimising the right-hand side yields an alternative diffusion alignment objective,
\begin{align} \label{eq:vmpo-difference-potential-obj}
    \argmin_\theta \sum_t \mathbb{V}_{h} \left( \log \frac{\exp (r(x_{t-1}) / \beta)}{\exp (r(x_{t}) / \beta)} \frac{\pref (x_{t-1} | x_t)}{p_\theta (x_{t-1} | x_t)} \right).
\end{align}
which leads to the following practical loss:
\begin{align}
    \argmin_{\theta, \phi} \E_{t, h} \left| \log \frac{\exp (r(x_{t-1}) / \beta)}{\exp (r(x_{t}) / \beta)} \frac{\pref (x_{t-1} | x_t)}{p_\theta (x_{t-1} | x_t)} - M_\phi (t) \right|^2.
\end{align}
Notably, this objective is equivalent to the diffusion alignment formulation proposed in \cite{ou2025inference}, which focuses on discrete diffusion models. In practice, rewards are often only defined on clean data $x_0$. 
In such cases, following \cite{wu2023practical}, intermediate rewards can be constructed as
\begin{align}
    r(x_t) = \E_{p_\theta (x_0 | x_t)} [r(x_0)] \approx r(\hat{x}_\theta (x_t, t)).
\end{align}
More generally, intermediate rewards can be parameterised using a neural network $F_\phi: \mathcal{X} \rightarrow \mathbb{R}$:
\begin{align}
    r_\phi (x_t) = F_\phi (x_t) r(\hat{x}_\theta (x_t, t)), \quad F_\phi (x_0) = 1
\end{align}
where the constraint $F_\phi (x_0) = 1$ ensures satisfaction of the marginal constraint in \cref{eq:potential-marginal-constraint}.
This parameterisation is closely related to the forward-looking reward shaping used in GFlowNets \citep{pan2023better}.

\textbf{Remark.} 
The return-to-go and difference potentials provide two valid decompositions of the same globally tilted target distribution. 
The return-to-go formulation assigns the terminal reward uniformly across timesteps, resulting in a simple objective that does not require intermediate rewards, but offers limited temporal credit assignment.
In contrast, the difference potential yields a telescoping decomposition that attributes reward changes to individual transitions, enabling more localised credit assignment. This can provide a more targeted optimisation signal when reliable intermediate rewards are available, but relies on their accurate estimation in diffusion models.

\subsection{Connection to Existing Work} \label{sec:appendix-conect-related-work}
We then connect the proposed VMPO framework to existing works, demonstrating that a wide range of prior methods can be unified under VMPO through different choices of variance minimisation.

\textbf{GVPO \citep{zhang2025gvpo}.}
We first consider Group Variance Policy Optimisation (GVPO) \citep{zhang2025gvpo}, which is designed to maximise the reward of a large language model (LLM) $p_\theta (x | c)$, where $x$ denotes the generated response and $c$ the prompt.
GVPO optimises the following objective \citep[Equantion 8]{zhang2025gvpo}
\begin{align} \label{eq:gvpo}
    \argmax_\theta \sum_{i=1}^K \frac{p_\theta (x^i | c)}{p_{\theta_{\mathrm{old}}} (x^i | c)} A^i, 
    A^i = \frac{1}{\beta} r(x^i | c) + \log \frac{p_{\theta_{\mathrm{sg}}} (x^i | c)}{p_{\theta_{\mathrm{old}}} (x^i | c)} - \frac{1}{\beta} \overline{r(x | c)} - \overline{\log \frac{p_{\theta_{\mathrm{sg}}} (x | c)}{p_{\theta_{\mathrm{old}}} (x | c)}},
\end{align}
where $\overline{r(x | c)} = \frac{1}{K} \sum_{i=1}^K r(x^i | c)$ and $\overline{\log \frac{p_{\theta_{\mathrm{sg}}} (x | c)}{p_{\theta_{\mathrm{old}}} (x | c)}} = \frac{1}{K} \sum_{i=1}^K \log \frac{p_{\theta_{\mathrm{sg}}} (x^i | c)}{p_{\theta_{\mathrm{old}}} (x^i | c)}$. 
It can be seen that the objective in \cref{eq:gvpo} closely resembles the VMPO formulation in \cref{eq:appendix-vamp3-obj}, differing primarily in the choice of model class and the tractability of likelihood evaluations. While GVPO has demonstrated strong empirical performance for LLM alignment, it does not directly extend to diffusion models, as the likelihood $p_\theta (x|c)$ is intractable in diffusion models but tractable for autoregressive LLMs. To address this limitation, we derive an upper bound on the log-variance of the importance weights, which yields a tractable VMPO objective suitable for diffusion alignment, as formalised in \cref{eq:appendix-log-var-upperbound-r2g}.

\textbf{FlowRL \citep{zhu2025flowrl}.}
FlowRL is an alternative approach for reward optimisation in LLMs.
Specifically, it train the model $p_\theta (x | c)$ using the following objective \citep[Equation 6]{zhu2025flowrl}
\begin{align} \label{eq:flowrl-obj}
    \argmin_{\theta, \phi} w \cdot \left( \log Z_\phi (x) + \frac{1}{|x|} \log p_\theta (x|c) - \frac{1}{|x|} \log \pref (x|c) - r(x|c) \right)^2,
\end{align}
where $w = \mathrm{clip}\left(\frac{p_{\theta_{\mathrm{sg}}} (x|c)}{p_{\theta_{\mathrm{old}}} (x|c)}, 1-\epsilon, 1+\epsilon\right)$ denotes the clipped importance weight, and $Z_\phi$ the learnable partition function.
This objective can be interpreted as a VMPO variant in which a neural estimator is used to approximate the mean of the log importance weight, as in \cref{eq:vmpo-amor-z-obj}. As with GVPO, FlowRL relies on explicit likelihood evaluations and is therefore not directly applicable to diffusion models. VMPO circumvents this limitation by operating on variance bounds of importance weights at each time step, enabling diffusion alignment without requiring tractable likelihoods of the clean data.

\textbf{GFlowNet \citep{zhang2024improving}.}
Under the difference potential, VMPO induces the per-timestep importance weight as in \cref{eq:vmpo-difference-potential-obj}
\begin{align}
    w_t = \frac{\exp (r(x_{t-1}) / \beta)}{\exp (r(x_{t}) / \beta)} \frac{\pref (x_{t-1} | x_t)}{p_\theta (x_{t-1} | x_t)}.
\end{align}
Beyond directly minimising the variance of importance weights, 
an alternative way to control the variance of these weights is to directly enforce consistency, via mean square errors, between the numerator and denominator at each timestep, yielding the objective
\begin{align} \label{eq:vmpo-gflownet-obj}
    \argmin_\theta \sum_t \left( \log \frac{\exp (r(x_{t-1}) / \beta)}{\exp (r(x_{t}) / \beta)} \frac{\pref (x_{t-1} | x_t)}{p_\theta (x_{t-1} | x_t)} \right)^2,
\end{align}
whose minimum of is attained when $w_t = 1$ and thus $\mathbb{V}(\log w_t) = 0$.
The objective in \cref{eq:vmpo-gflownet-obj} closely resembles the detailed-balance objectives used in GFlowNets \citep{zhang2024improving}, with the main difference arising from how intermediate rewards are parameterised. In \cite{zhang2024improving}, the intermediate reward is modelled via a forward-looking estimator $\Tilde{r}_\phi(x_t) = F_\phi (x_t) r(\hat{x}_\theta (x_t, t))$, subject to the terminal constraint $\Tilde{r}_\phi(x_0) = r(x_0)$. This yields the following loss:
\begin{align} \label{eq:vmpo-gflownet-obj2}
    \argmin_{\theta, \phi} \sum_t \left( \log \frac{\exp (\Tilde{r}_\phi (x_{t-1}) / \beta)}{\exp (\Tilde{r}_\phi (x_{t}) / \beta)} \frac{\pref (x_{t-1} | x_t)}{p_\theta (x_{t-1} | x_t)} \right)^2.
\end{align}
Due to the terminal constraint, the optimal policy still satisfies $p_{\theta^*} \propto \pref (x_{t-1} | x_t) \exp (r(x_{t}) / \beta)$, consistent with the analysis of the difference potential in \cref{sec:appendix-vmpo-kaleidoscopes}.
Notably, \cite{zhang2024improving} replace $\pref$ with the forward noising kernel $q(x_t | x_{t+1})$ in \cref{eq:vmpo-gflownet-obj2}, which leads to the optimal policy $p_{\theta^*} \propto \exp (r(x_{t}) / \beta)$. We refer to \cite{liu2024efficient} for further details.

\textbf{$\nabla$-GFlowNet \citep{liu2024efficient}.}
Beyond the above methods, an alternative route to zero variance is to enforce the log importance weight to be constant. A sufficient condition for this is the vanishing of its gradient, since
\begin{align}
    \lVert \nabla_x \log w_t  \rVert_2^2 = 0 \ \Rightarrow \ \log w_t = const. \ \Rightarrow \ \mathbb{V}(\log w_t) = 0.
\end{align}
This observation motivates learning the policy by penalising the gradient norm of the log importance weight.
Under the difference potential, this leads to the following objective:
\begin{align} \label{eq:vmpo-nabla-obj1}
    \argmin_\theta \sum_t \left\lVert \frac{1}{\beta} \nabla_{x_{t-1}} r(x_{t-1}) +  \nabla_{x_{t-1}} \log \pref (x_{t-1} | x_t) - \nabla_{x_{t-1}} \log p_\theta (x_{t-1} \mid x_t) \right\rVert_2^2.
\end{align}
This objective can be interpreted as minimising the Fisher divergence between the learned policy $p_\theta (x_{t-1} | x_t)$ and the locally tilted target $\ptilt (x_{t-1} | x_t) \propto \pref (x_{t-1} | x_t) \exp(r(x_{t-1} / \beta))$.
An equivalent formulation can be obtained by differentiating with respect to $x_t$, yielding
\begin{align}  \label{eq:vmpo-nabla-obj2}
    \argmin_\theta \sum_t \left\lVert \nabla_{x_{t}} \log \pref (x_{t-1} | x_t) - \nabla_{x_{t}} \log p_\theta (x_{t-1} \mid x_t) - \frac{1}{\beta} \nabla_{x_{t}} r(x_{t})   \right\rVert_2^2.
\end{align}
In practice, the policy $p_\theta$ may be trained using either of these objectives, or a combination of both.
These gradient-based objectives closely resemble those proposed in $\nabla$-GFlowNet \citep{liu2024efficient}, differing mainly in the parametrisation of intermediate rewards. In \cite{liu2024efficient}, the intermediate reward is modelled via a forward-looking parametrisation with a terminal constraint, as in \cref{eq:vmpo-gflownet-obj2}.
Consequently, the reward gradient takes the form
\begin{align}
    \nabla_{x_{t}} \Tilde{r}_\phi (x_{t}) = \nabla_{x_t} r(\hat{x}_\theta (x_t, t)) + \beta F_\phi (x_t), \qquad F_\phi (x_0) = 0.
\end{align}
Substituting this parametrisation into the gradient-matching objectives yields
\begin{align}
    &\!\!\argmin_{\theta, \phi} 
    \!\!\sum_t \!\left\lVert \frac{1}{\beta} \nabla_{\!x_{t\!-\!1}}\! r(\hat{x}_\theta (x_{t\!-\!1}, t\!-\!1)) \!+\! \nabla_{\!x_{t\!-\!1}}\! F_\phi (x_{t\!-\!1}) \!+\!  \nabla_{\!x_{t\!-\!1}}\!\! \log \pref (x_{t\!-\!1} | x_t) \!-\! \nabla_{\!x_{t\!-\!1}}\!\! \log p_\theta (x_{t-1} \mid x_t) \right\rVert_2^2 \nonumber \\
    &\qquad + \left\lVert  \nabla_{x_{t}}\! \log \pref (x_{t\!-\!1} | x_t) \!-\! \nabla_{x_{t}}\! \log p_\theta (x_{t-1} \mid x_t) - \frac{1}{\beta} \nabla_{x_{t}}\! r(\hat{x}_\theta (x_t, t)) \!+\! \nabla_{x_{t}}\! F_\phi (x_{t}) \right\rVert_2^2,
\end{align}
which recovers the $\nabla$-GFlowNet objective defined in \citep[Equation 19]{liu2024efficient}.

\section{Additional Experimental Details and Results} \label{sec:appendix-additional-exp}

\subsection{Experimental Details}

\textbf{Settings for SD1.5.}
We finetune SD1.5 \citep{rombach2022high} using HPSv2 \citep{wu2023human} and ImageReward \citep{xu2023imagereward}.
For HPSv2, we adopt the photo and painting prompts from \cite{wu2023human} as training data, while for ImageReward we use the DrawBench prompt set \citep{saharia2022photorealistic}. For evaluation, we follow \cite{domingo2024adjoint} and use their collected prompt set as test prompts. Both training setups largely share the same hyperparameter configuration, following the DDPO \citep{black2023training} implementation; full details are provided here for completeness.

We use LoRA \citep{hu2022lora} with $\alpha=4$ and $r=4$.
Trainings are performed on 2 NVIDIA A100 80GB GPUs with a per-GPU batch size of $8$. With 4-step gradient accumulation, this yields an effective batch size of $64$.
We train for $150$ epochs, where each epoch consists of sampling $128$ trajectories from $p_{\theta_{\mathrm{old}}}$ using a 50-step DDIM schedule \citep{song2020denoising} during the rollout phase, and performing 2 optimisation steps.
The learning rate is fixed at $3\times 10^{-4}$ for both the diffusion model $p_\theta$ and the mean estimator $M_\phi$.
We employ the AdamW optimiser \citep{loshchilov2017decoupled} with gradient clipping at a norm of $1$.

During training, we adopt classifier-free guidance \citep{ho2022classifier} with a guidance scale of 5. Reward rescaling is employed to improve training efficiency. Specifically, we set $\beta=0.01$ for HPSv2 and $\beta=0.1$ for ImageReward to compute the loss in \cref{eq:vmpo-amor-z-obj}.
We also incorporate a KL regularisation $\mathbb{KL}(p_\theta (x_{t-1} | x_t) || p_{\theta_{\mathrm{old}}}(x_{t-1} | x_t))$ with a coefficient of $1$ to enhance training stability, consistent with \cite{fan2023dpok}.
For the GRPO baseline, we adopt the implementation\footnote{\url{https://github.com/yifan123/flow_grpo}} from \cite{liu2025flow} with a group size of $8$, yielding the same effective number of samples as in our setup. All other hyperparameters are kept identical to those used for VMPO, as described above, to ensure a fair comparison.

\begin{table}[!t] 
    \centering
    \caption{Results on Stable Diffusion v1.5 fine-tuned with ImageReward.}
    \label{tab:finetuning-results}
    \vspace{-2mm}
    \begin{tabular}{lcccc}
        \toprule 
        \textbf{Method} & \textbf{ImageReward} ($\uparrow$) & \textbf{CLIPScore} ($\uparrow$) & \textbf{HPSv2} ($\uparrow$) & \textbf{DreamSim} ($\uparrow$) \\
        \midrule 
        SD1.5 (Base) &
        0.0331 ± 0.0779 & 0.2717 ± 0.0032	& 0.2368 ± 0.0029	& 0.4389 ± 0.0116 \\
        GRPO &  
        0.2490 ± 0.0792 & 0.2720 ± 0.0030	   & 0.2420 ± 0.0032	 & 0.3755 ± 0.0117 \\
        VMPO-R2G &  
        0.1698 ± 0.0758 & 0.2721 ± 0.0032	   & 0.2423 ± 0.0030	 & 0.4008 ± 0.0117 \\
        VMPO-Diff &  
        0.4716 ± 0.0740 & 0.2624 ± 0.0030	   & 0.2645 ± 0.0031	 & 0.3005 ± 0.0117 \\
        \bottomrule 
    \end{tabular}
\end{table}

\textbf{Settings for SD3.5-M.}
We consider the text rendering task following \cite{liu2025flow}, which fine-tunes SD3.5-M to improve the visual text rendering capability of diffusion models. The model is trained using the training prompt set provided in \cite{liu2025flow}, and OCR accuracy is evaluated on their corresponding test prompt set. To assess potential reward hacking beyond task-specific accuracy, we additionally evaluate five complementary metrics on the DrawBench prompt set \citep{saharia2022photorealistic}, following \cite{liu2025flow}: Aesthetic Score\footnote{\url{https://laion.ai/blog/laion-aesthetics/}}, DeQA \citep{you2025teaching}, ImageReward \citep{xu2023imagereward}, PickScore \citep{kirstain2023pick}, and UnifiedReward \citep{wang2025unified}.

Following \cite{liu2025flow}, we use LoRA with $\alpha=64$ and $r=32$. The model is trained on 4 NVIDIA A100 80 GPUs with a batch size of 4. We apply 4-step gradient accumulations, yielding an effective batch size of 64. For each training epoch, we sample $128$ trajectories from $p_{\theta_{\mathrm{old}}}$ using $10$ steps and perform 2 optimisation steps. While at test time, the number of sampling steps is set to $40$, following the setting in \cite{liu2025flow}.
We employ the AdamW optimiser \citep{loshchilov2017decoupled} with the learning rate fixed as $3\times 10^{-4}$ for both $p_\theta$ and $M_\phi$. Moreover, we use bf16 for $p_\theta$ while float32 for $M_\phi$, which is important to stabilise training. During training, we adopt classifier-free guidance with a guidance scale of 4.5 and incorporate a KL regularisation $\mathbb{KL}(p_\theta (x_{t-1} | x_t) || p_{\theta_{\mathrm{old}}}(x_{t-1} | x_t))$ with a coefficient of $1$ to enhance training stability. 
We do not train GRPO ourselves; instead, we evaluate the released checkpoint\footnote{\url{https://huggingface.co/jieliu/SD3.5M-FlowGRPO-Text}} provided by \cite{liu2025flow}. Notably, this checkpoint was trained with a larger batch size and on more GPUs than our setup.

\subsection{Additional Experimental Results}

\begin{wrapfigure}{r}{0.45\linewidth}
    \centering
    \vspace{-2mm}
    \includegraphics[width=.98\linewidth]{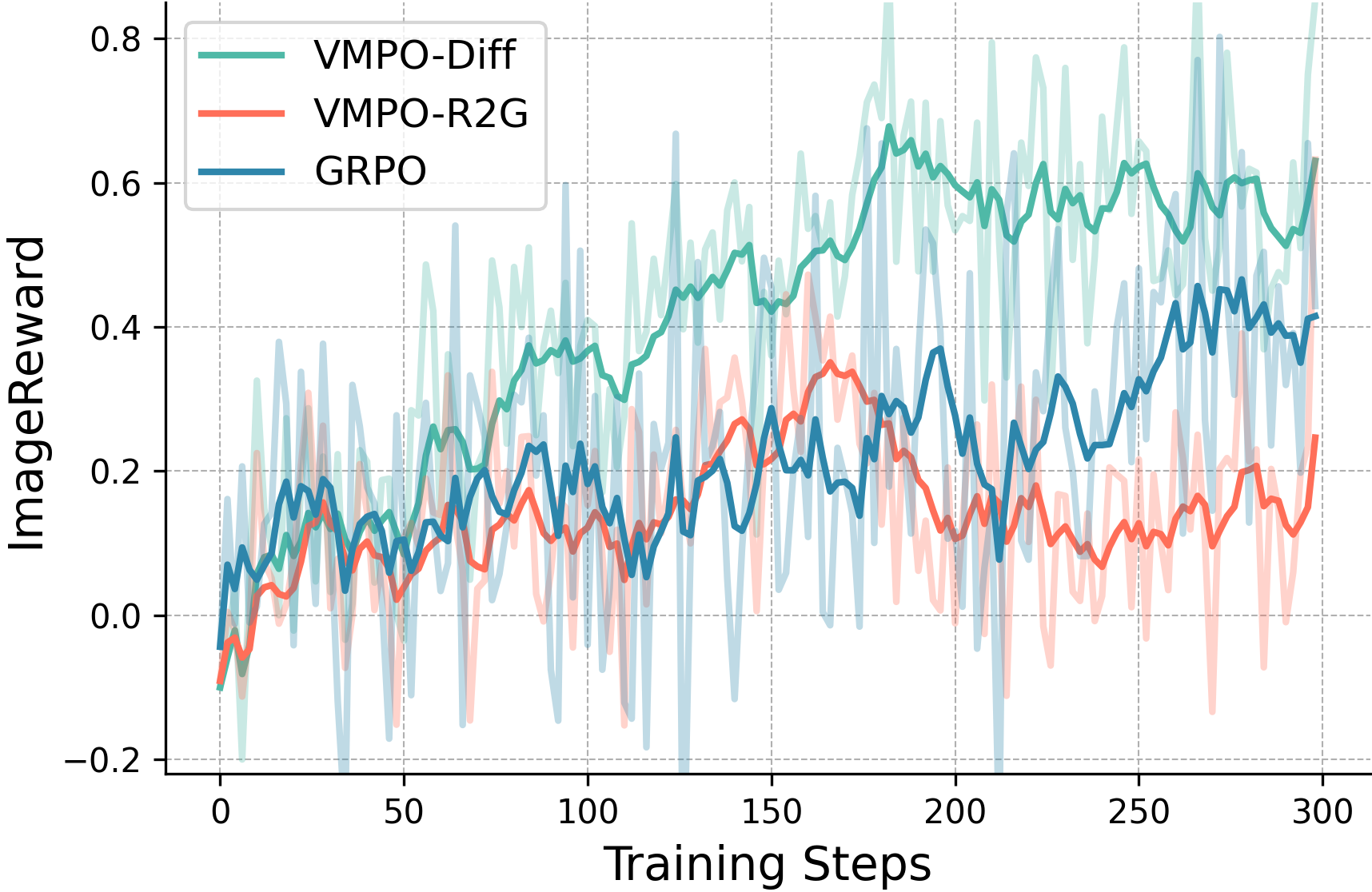}
    \vspace{-2mm}
    \caption{ImageReward convergence curves of SD1.5.}
    \label{fig:imagereward_imagereward_curve}
    \vspace{-3mm}
\end{wrapfigure}
\textbf{Aligning SD1.5 with ImageReward.}
We further evaluate the performance of VMPO by fine-tuning SD1.5 using ImageReward. 
As illustrated by the convergence curves in \cref{fig:imagereward_imagereward_curve}, VMPO-Diff demonstrates a significantly higher reward ceiling and faster optimisation compared to both GRPO and VMPO-R2G. This is quantitatively confirmed in Table 2, where VMPO-Diff achieves a superior ImageReward score $0.4716 \pm 0.0740$, more than doubling the performance of the GRPO baseline.
While all methods show substantial gains over the base SD1.5 model in the target reward, we observe a similar trade-off as seen in the HPSv2 experiments: the increase in ImageReward is accompanied by a decrease in DreamSim scores. 
This suggests that while the model is successfully aligning with human preferences represented by ImageReward, it does so at the cost of visual diversity, a hallmark of reward hacking.

\begin{table*}[!t]
    \centering
    \renewcommand{\arraystretch}{.8}
    \caption{Visual text rendering results on Stable Diffusion v3.5-M. trained with OCR Accuracy.}
    \vspace{-2mm}
    \resizebox{\linewidth}{!}{
        \begin{tabular}{lcccccc}
            \toprule
            \multirow{2}{*}{\textbf{Model}} & \multicolumn{1}{c}{\textbf{Task Metric}} & \multicolumn{2}{c}{\textbf{Image Quality}} & \multicolumn{3}{c}{\textbf{Preference Score}}    \\ \cmidrule(lr){2-2} \cmidrule(lr){3-4} \cmidrule(l){5-7} 
                                   & \textbf{OCR Acc.} & \textbf{Aesthetic}          & \textbf{DeQA}         & \textbf{ImgRwd} & \textbf{PickScore} & \textbf{UniRwd} \\ \midrule
            SD3.5-M & 0.59 & 5.38 & 4.08 & 0.84 & 22.42 & 3.08 \\
            Flow-GRPO  & 0.92 & 5.32 & 4.10 & 0.95 & 22.50 & 3.12 \\
            \rowcolor{gray!20} VMPO-Diff  & 0.91 & 5.25 & 4.07 & 0.94 & 22.43 & 3.08 \\
            \bottomrule
            \end{tabular}
    }
    \vspace{-2mm}
    \label{tab:ocr_result}
\end{table*}

\textbf{Improving Visual Text Rendering of SD3.5-M.}
Given the superior performance of VMPO-Diff in aligning SD1.5, we further evaluate its scalability by optimising the visual text rendering capabilities of the larger SD3.5-M model using OCR accuracy as the reward signal.
As shown in \cref{fig:ocr_curve}, VMPO consistently improves OCR accuracy throughout the training process, ultimately achieving a score of 0.91 (see \cref{tab:ocr_result}). 
This represents a significant leap over the base SD3.5-M model's score of 0.59, 
\begin{wrapfigure}{r}{0.45\linewidth}
    \centering
    \vspace{-2mm}
    \includegraphics[width=.98\linewidth]{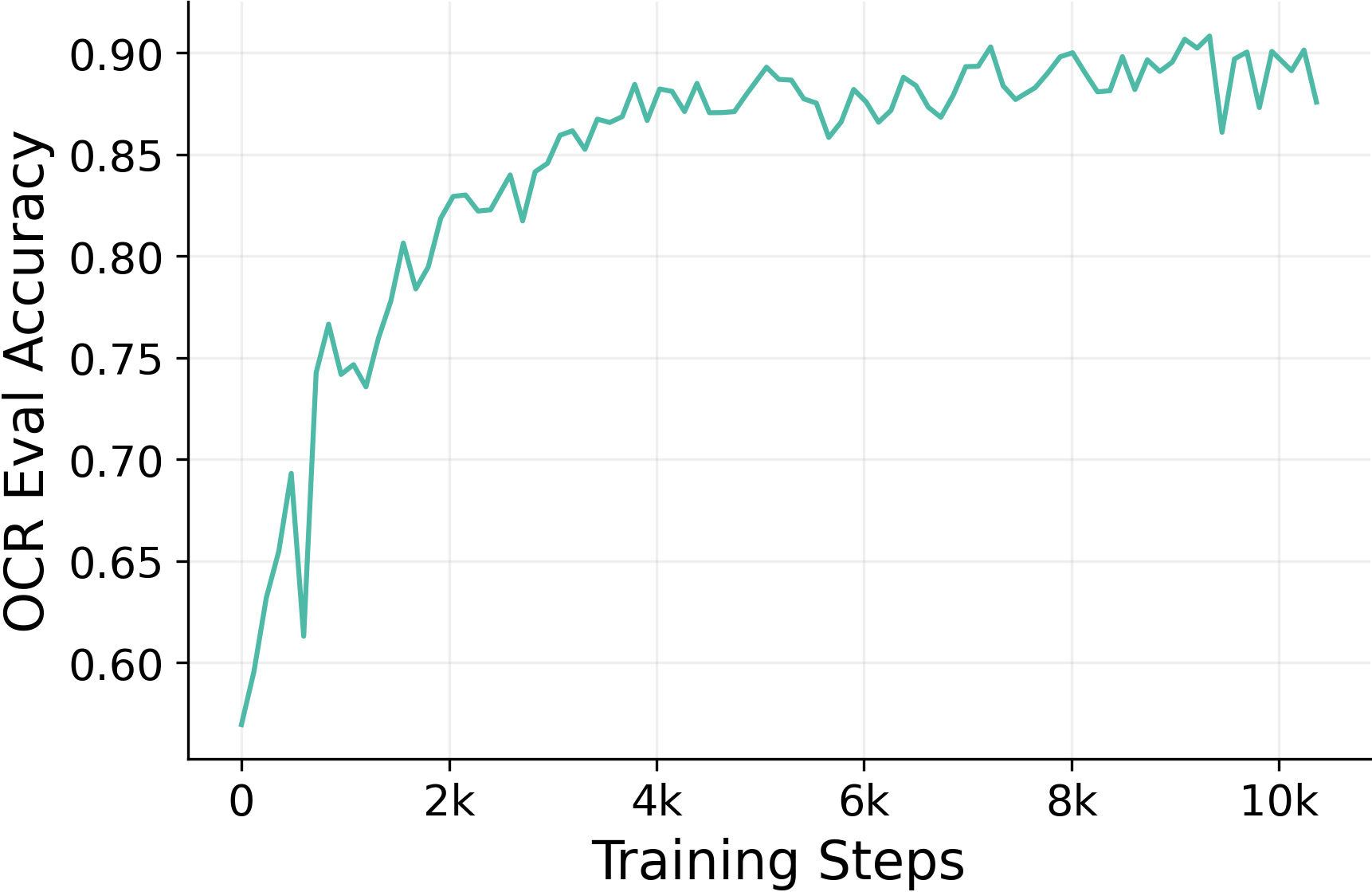}
    \vspace{-2mm}
    \caption{OCR accuracy convergence curves of SD3.5-M.}
    \label{fig:ocr_curve}
    \vspace{-3mm}
\end{wrapfigure}
demonstrating the effectiveness of our alignment method even in large-scale models.
Notably, VMPO achieves performance comparable to Flow-GRPO (0.92), despite Flow-GRPO requiring significantly larger batch sizes and more extensive GPU resources for training.
Qualitative results in \cref{fig:ocr-samples-comp} further validate these findings: while the base model often struggles with legibility and character consistency, VMPO generates images with sharp, accurate, and well-aligned text.
Despite a marginal decrease in image quality compared to the base model, the drastic improvement in task-specific accuracy highlights VMPO as a highly effective alternative for specialised alignment tasks.
A large body of work aims to reduce the number of sampling steps required by diffusion models, enabling fast or even single-step generation.

\section{Related Work}

\textbf{Diffusion and Flow Models.}
Diffusion \citep{sohl2015deep,ho2020denoising,song2020score} and flow models \citep{liu2022flow,lipman2022flow,albergo2023stochastic} have achieved remarkable success in modelling complex data across various domains, including image synthesis \citep{rombach2022high}, 3D generation \citep{poole2022dreamfusion}, video synthesis \citep{ho2022video}, and audio generation \citep{liu2023audioldm}.
Beyond improving sample quality, recent research has increasingly focused on generation efficiency.
A large body of work aims to reduce the number of sampling steps required by diffusion models, enabling fast or even single-step generation.
Representative approaches include accelerated solvers \citep{song2020denoising,lu2022dpm}, more expressive posterior parametrisation \citep{xiao2021tackling,ou2024improving}, consistency and distillation-based methods \citep{salimans2022progressive,luo2023diff,zhou2024score,zhang2025towards,wang2025vardiu,chen2025diffratio}.
In parallel, aligning diffusion models with human preferences or task-specific objectives has emerged as an important research direction.
Early work focuses on classifier guidance \citep{dhariwal2021diffusion} and classifier-free guidance \citep{ho2022classifier}, which steer generation using auxiliary models at inference time.
Additionately, alternative methods formulate alignment as a probability inference problem, casting alignment as sampling from a reward-tilted distribution \citep{he2025crepe,zhang2025accelerated}. 
In our paper, we consider more recent methods that directly incorporate preference learning into training.

\textbf{Reward Finetuning for Diffusion Alignment.}
Beyond the guidance methods, recent efforts on diffusion alignment focus on finetuning pretrained models to maximise expected reward.
One line of research directly performs optimisation by backpropagating through the entire sampling trajectory \citep{clark2023directly, prabhudesai2023aligning}, which requires differentiable reward signals.
Another class of methods reframes alignment as a likelihood-based objective, applying maximum likelihood estimation to model-generated samples that are reweighted according to their rewards \citep{lee2023aligning, dong2023raft}.
Preference-based learning has also been adapted to diffusion models by extending direct preference optimisation with paired human preferences data \citep{rafailov2023direct}. However, unlike autoregressive models with tractable likelihoods, diffusion models require additional approximations for likelihood evaluation and loss approximations \citep{wallace2024diffusion, liang2024step, yuan2024self}.
A separate strand of work interprets the diffusion sampling procedure as a Markov decision process and applies reinforcement learning techniques to perform alignment \citep{black2023training, fan2023dpok, liu2025flow, xue2025dancegrpo, li2025mixgrpo}.
Beyond standard RL formulations, alternative perspectives have been explored, including stochastic optimal control \citep{uehara2024fine, domingo2024adjoint, potaptchik2026meta, potaptchik2025tilt} and GFlowNet-based objectives \citep{zhang2024improving, liu2024efficient}.
More recently, DiffusionNFT \citep{zheng2025diffusionnft} demonstrates strong sample efficiency by eliminating the need for likelihood estimation and SDE-based reverse process.
In contrast to these approaches, the proposed VMPO method offers a fundamentally different viewpoint on diffusion alignment, grounding optimisation in Sequential Monte Carlo and explicitly minimising the variance of importance weights to better align the model with the reward function.

\textbf{Sequential Monte Carlo for Generative Modelling.}
SMC \citep{del2006sequential} constitutes a well-established class of probabilistic inference techniques, providing adaptable and efficient sampling strategies across a variety of settings, including Bayesian experimental design \citep{ryan2016review} and particle filtering \citep{johansen2009tutorial}.
More recently, SMC has been combined with diffusion-based generative models \citep{chen2025neural, ou2025discrete, he2025rne, skreta2025feynman, wu2025reverse}, resulting in neural sampling frameworks that enable effective approximation of complex Boltzmann distributions. These hybrid approaches have been successfully applied to posterior inference \citep{dou2024diffusion, cardoso2023monte} as well as reward-guided alignment \citep{wu2023practical, singhal2025general}, all without requiring additional model training.
Beyond diffusion models, SMC has also been explored as a tool for improving large language models. In particular, \citet{zhao2024probabilistic} propose SMC as a principled inference framework to mitigate capability and safety challenges in LLMs, a direction that has since been extended to test-time scaling for mathematical reasoning \citep{feng2024step, puri2025probabilistic}.
Our work is most closely related to recent efforts that learn optimal proposal distributions for SMC in discrete diffusion models \citep{ou2025inference}.
While it also adopts variance minimisation as a learning objective, VMPO offers a more comprehensive study and understanding of variance minimisation, revealing connections to various previous alignment methods through different choices of potential functions and variance minimisation objectives.

\section{Limitation and Future Work} \label{sec:appendix-limitation-future-work}

As established in \cref{sec:appendix-holistic-vmpo}, VMPO provides a holistic framework that unifies various diffusion alignment methods through different potential functions and variance minimisation objectives.
This discloses several avenues for further exploration. 
One limitation of this paper is that we do not exhaustively discuss which specific variant serves as the optimal choice across different scenarios. 
While VMPO-Diff demonstrated strong empirical results, the design space for potentials is vast, and future research could systematically benchmark these variants against a broader range of baselines to identify their respective strengths in different alignment scenarios.

Furthermore, our empirical evaluation primarily focuses on text-to-image models like Stable Diffusion v1.5 and v3.5-M. Expanding this framework to other modalities, such as video generation \citep{kong2024hunyuanvideo,wan2025wan}, and to more recently advanced models, like Qwen-Image \citep{wu2025qwen}, Z-image \citep{cai2025z}, presents a promising direction for future work.
Additionally, our results indicate that VMPO, like other alignment methods, remains susceptible to reward hacking, as gains in preference scores often coincide with a reduction in visual diversity. Future research could investigate more robust regularisation strategies beyond standard KL penalties to better preserve the original data distribution while achieving effective reward alignment.
Finally, applying VMPO to align one-step generative models can be an interesting direction for future research.

\newpage

\begin{figure}[!t]
    \vspace{-2mm}
    \centering
    \begin{minipage}{.99\linewidth}
        \centering
        \begin{minipage}{.245\linewidth}\centering SD-1.5 \end{minipage}
        \hfill
        \begin{minipage}{.245\linewidth}\centering GRPO \end{minipage}
        \hfill
        \begin{minipage}{.245\linewidth}\centering VMPO-R2G \end{minipage}
        \hfill
        \begin{minipage}{.245\linewidth}\centering VMPO-Diff \end{minipage}
    \end{minipage}
    \begin{minipage}{0.99\linewidth}
        \centering
        \begin{minipage}{0.245\linewidth}
            \centering
            \includegraphics[width=.99\linewidth]{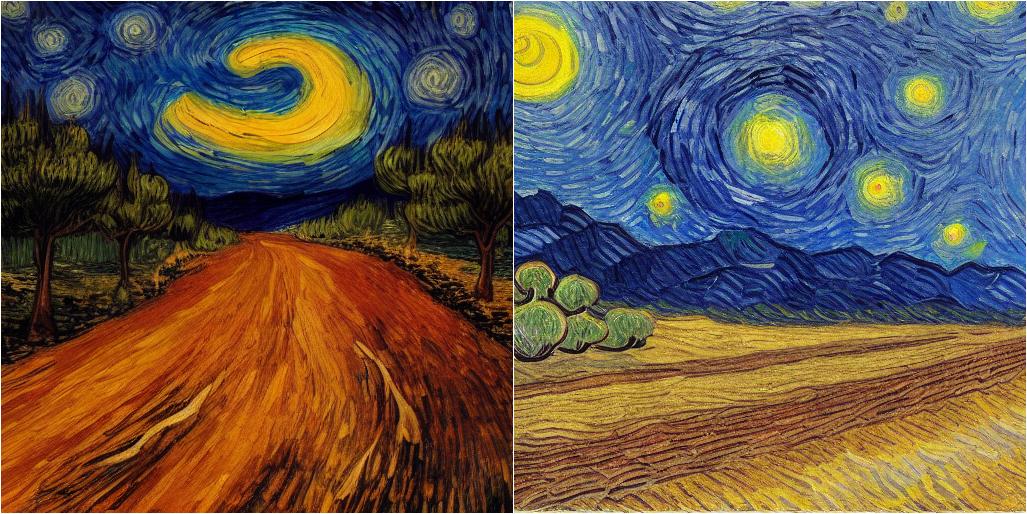}
        \end{minipage}
        \hfill
        \begin{minipage}{0.245\linewidth}
            \centering
            \includegraphics[width=.99\linewidth]{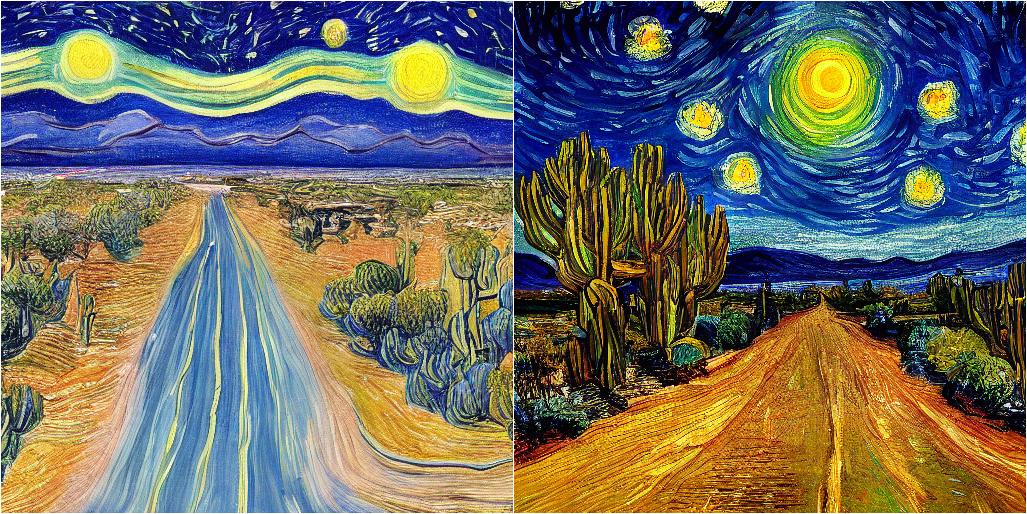}
        \end{minipage}
        \hfill
        \begin{minipage}{0.245\linewidth}
            \centering
            \includegraphics[width=.99\linewidth]{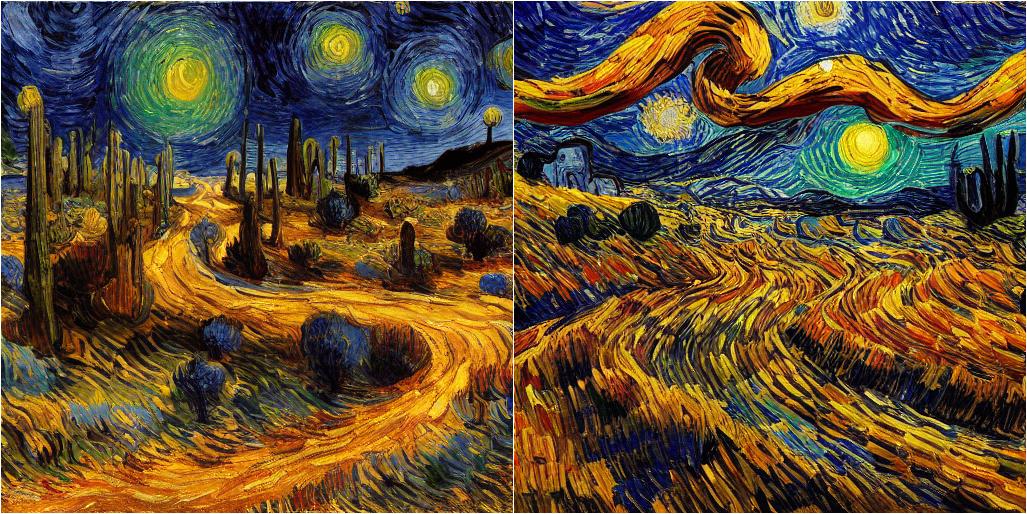}
        \end{minipage}
        \hfill
        \begin{minipage}{0.245\linewidth}
            \centering
            \includegraphics[width=.99\linewidth]{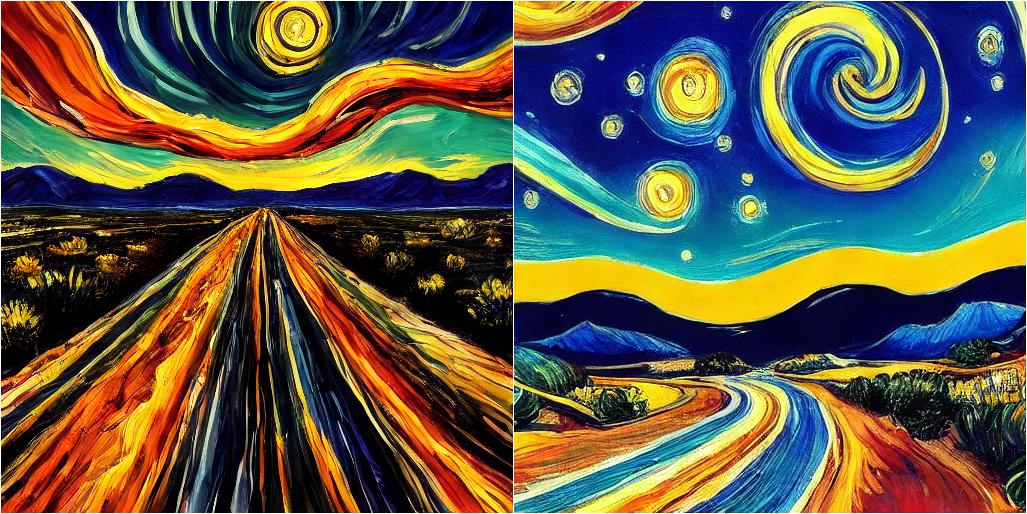}
        \end{minipage}
        \begin{minipage}{1.\linewidth}
            \smallskip
            \centering
            {\footnotesize A Van Gogh–style painting of a desert highway at night under a swirling starry sky.\par}
        \end{minipage}
    \end{minipage}
    \\ \vspace{.5mm}
    \begin{minipage}{0.99\linewidth}
        \centering
        \begin{minipage}{0.245\linewidth}
            \centering
            \includegraphics[width=.99\linewidth]{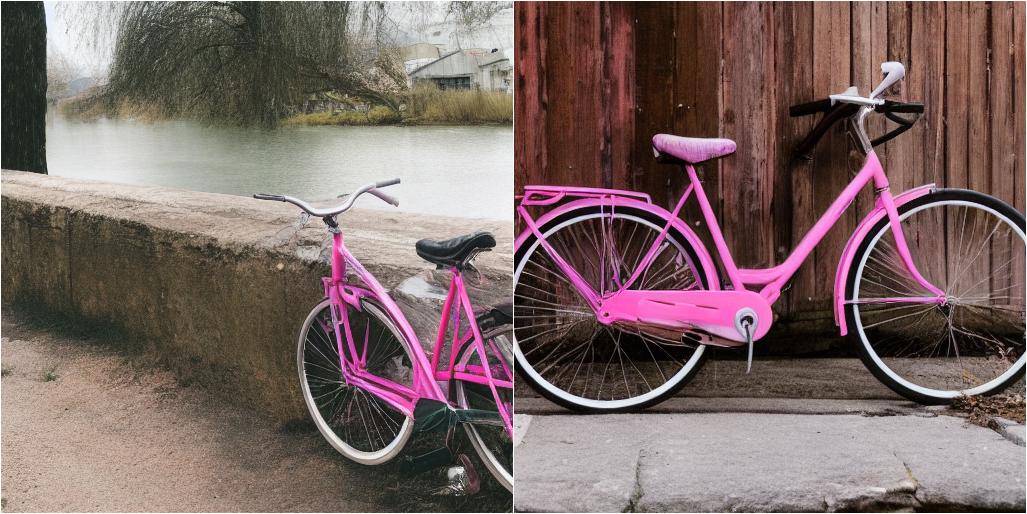}
        \end{minipage}
        \hfill
        \begin{minipage}{0.245\linewidth}
            \centering
            \includegraphics[width=.99\linewidth]{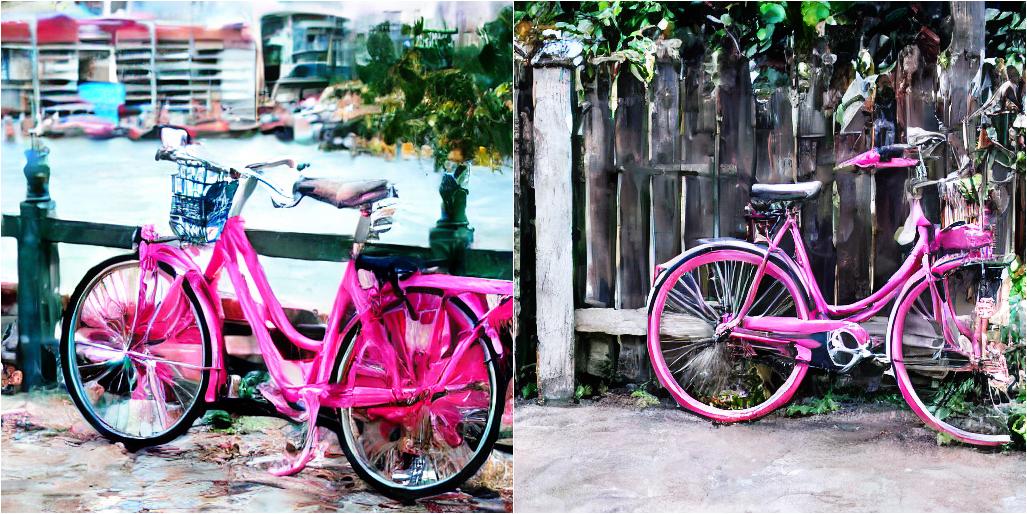}
        \end{minipage}
        \hfill
        \begin{minipage}{0.245\linewidth}
            \centering
            \includegraphics[width=.99\linewidth]{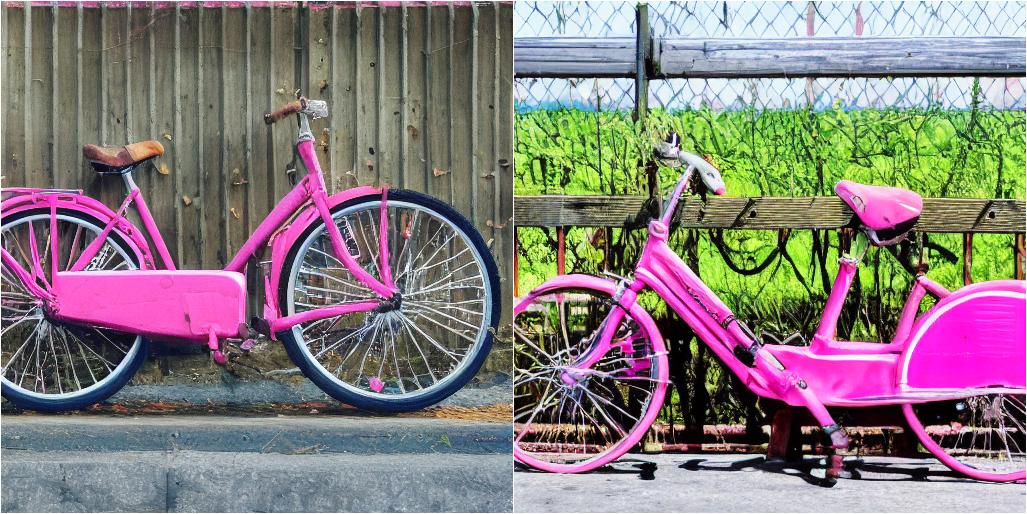}
        \end{minipage}
        \hfill
        \begin{minipage}{0.245\linewidth}
            \centering
            \includegraphics[width=.99\linewidth]{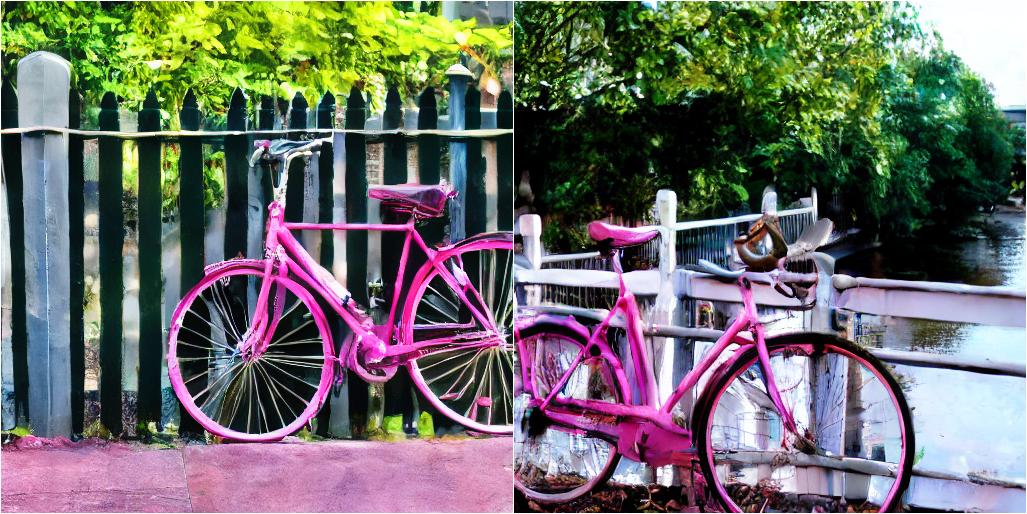}
        \end{minipage}
        \begin{minipage}{1.\linewidth}
            \smallskip
            \centering
            {\footnotesize A pink bicycle leaning against a fence near a river.\par}
        \end{minipage}
    \end{minipage}
    \\ \vspace{.5mm}
    \begin{minipage}{0.99\linewidth}
        \centering
        \begin{minipage}{0.245\linewidth}
            \centering
            \includegraphics[width=.99\linewidth]{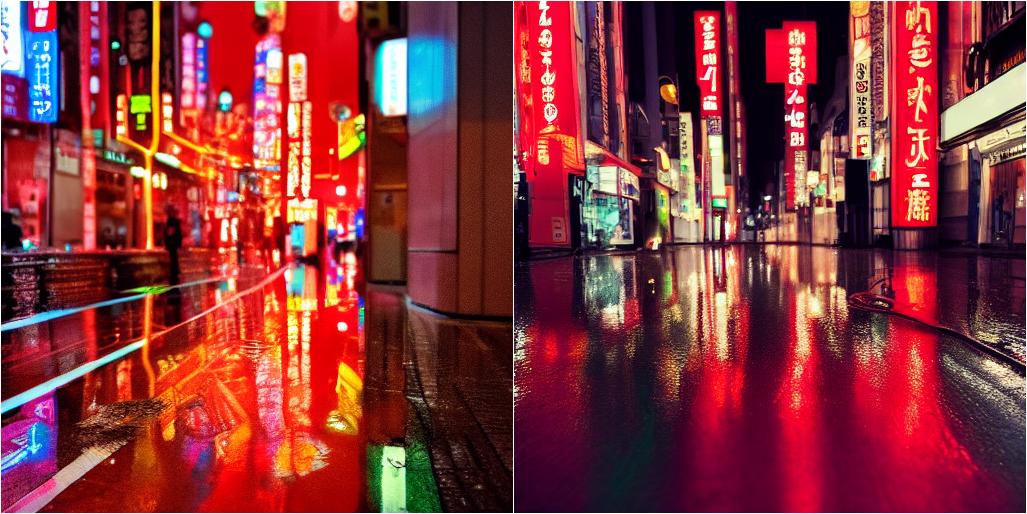}
        \end{minipage}
        \hfill
        \begin{minipage}{0.245\linewidth}
            \centering
            \includegraphics[width=.99\linewidth]{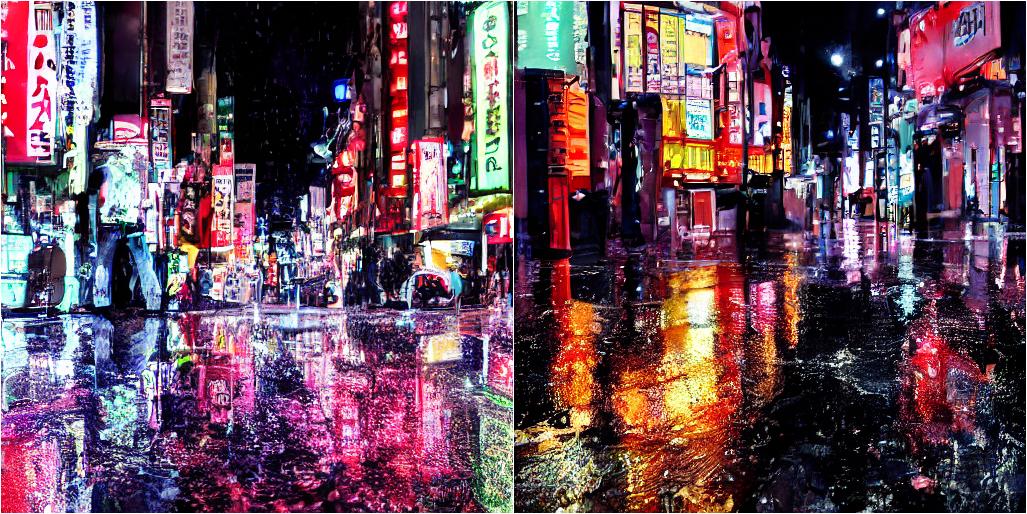}
        \end{minipage}
        \hfill
        \begin{minipage}{0.245\linewidth}
            \centering
            \includegraphics[width=.99\linewidth]{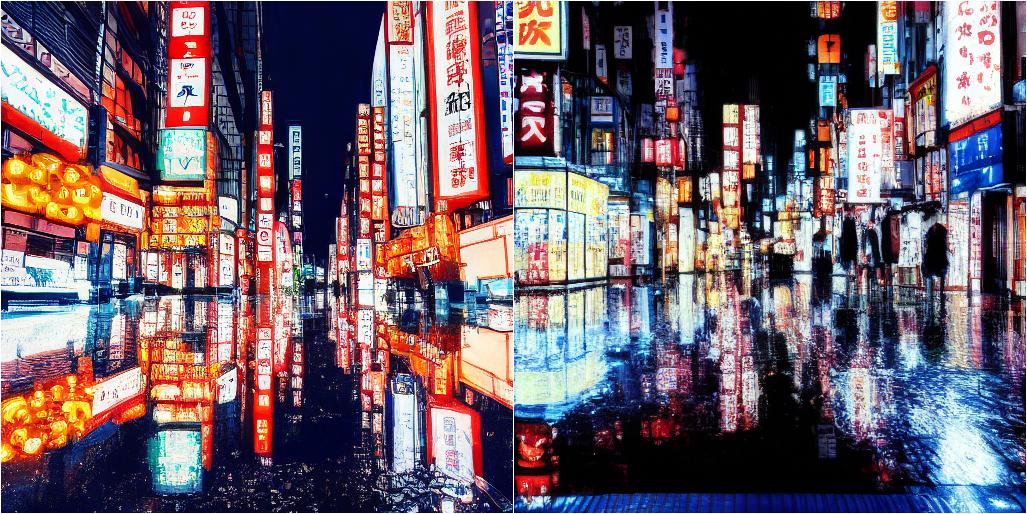}
        \end{minipage}
        \hfill
        \begin{minipage}{0.245\linewidth}
            \centering
            \includegraphics[width=.99\linewidth]{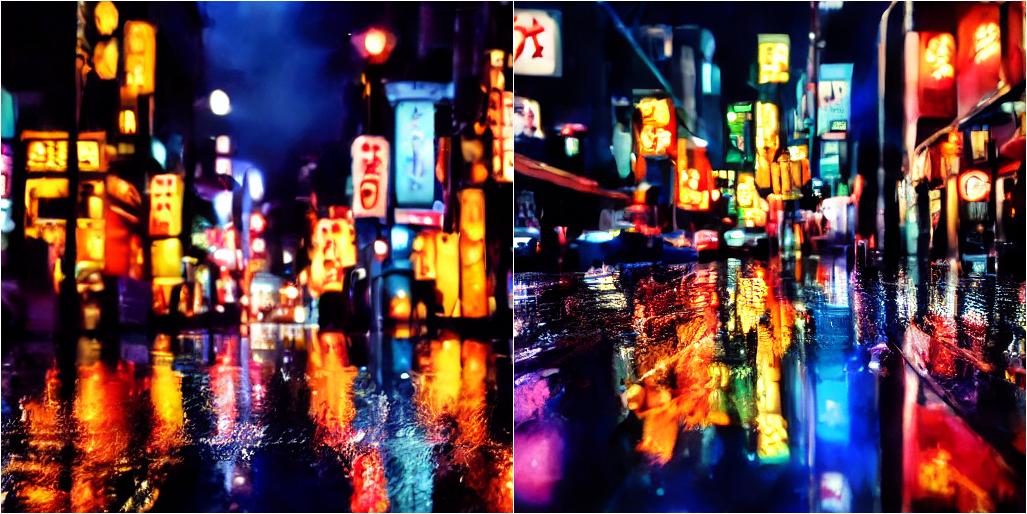}
        \end{minipage}
        \begin{minipage}{1.\linewidth}
            \smallskip
            \centering
            {\footnotesize A pink bicycle leaning against a fence near a river.\par}
        \end{minipage}
    \end{minipage}
    \\ \vspace{.5mm}
    \begin{minipage}{0.99\linewidth}
        \centering
        \begin{minipage}{0.245\linewidth}
            \centering
            \includegraphics[width=.99\linewidth]{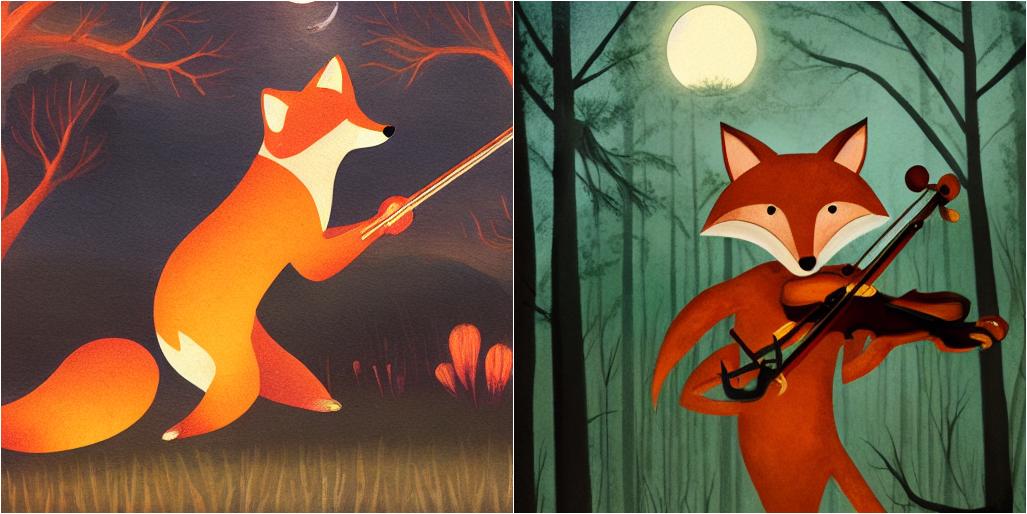}
        \end{minipage}
        \hfill
        \begin{minipage}{0.245\linewidth}
            \centering
            \includegraphics[width=.99\linewidth]{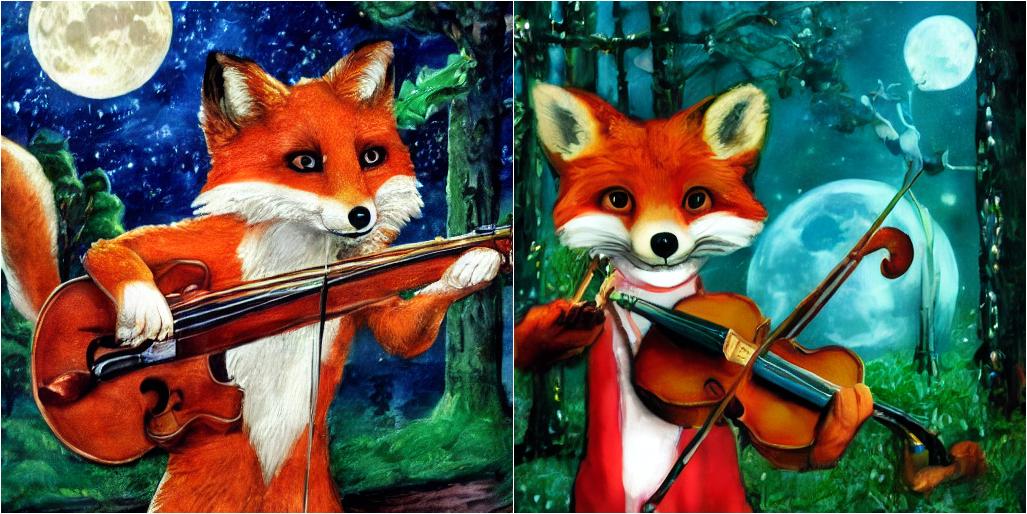}
        \end{minipage}
        \hfill
        \begin{minipage}{0.245\linewidth}
            \centering
            \includegraphics[width=.99\linewidth]{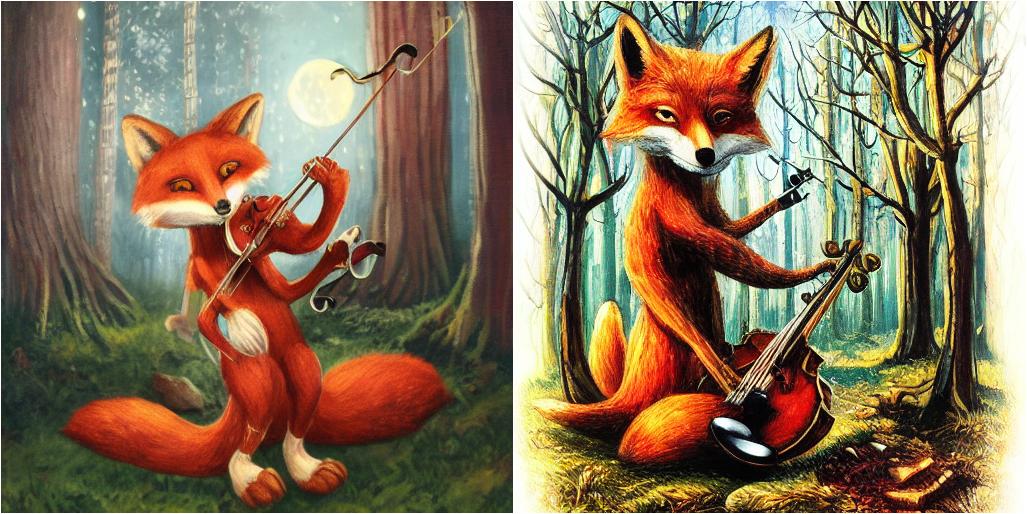}
        \end{minipage}
        \hfill
        \begin{minipage}{0.245\linewidth}
            \centering
            \includegraphics[width=.99\linewidth]{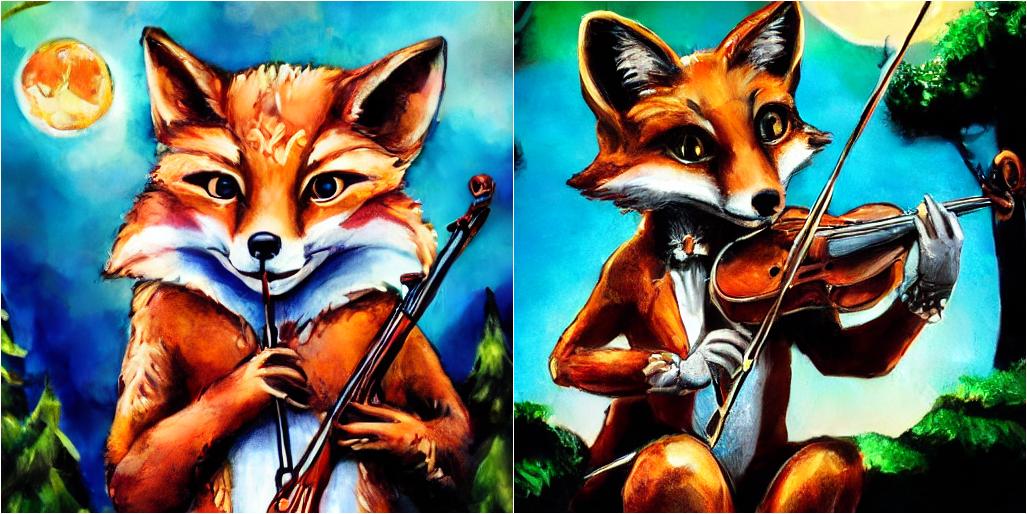}
        \end{minipage}
        \begin{minipage}{1.\linewidth}
            \smallskip
            \centering
            {\footnotesize An anthropomorphic fox playing a violin in a moonlit forest.\par}
        \end{minipage}
    \end{minipage}
    \vspace{-2mm}
    \caption{Illustration of the generated samples of different models.}
    \label{fig:hps-samples-comp-appendix}
    \vspace{-4mm}
\end{figure}

\begin{figure}[!t]
    \centering
    \begin{minipage}{.99\linewidth}
        \centering
        \begin{minipage}{.32\linewidth}\centering SD-3.5-M \end{minipage}
        \hfill
        \begin{minipage}{.32\linewidth}\centering GRPO \end{minipage}
        \hfill
        \begin{minipage}{.32\linewidth}\centering VMPO \end{minipage}
    \end{minipage}
    \begin{minipage}{0.95\linewidth}
        \centering
        \begin{minipage}{0.32\linewidth}
            \centering
            \includegraphics[width=.49\linewidth]{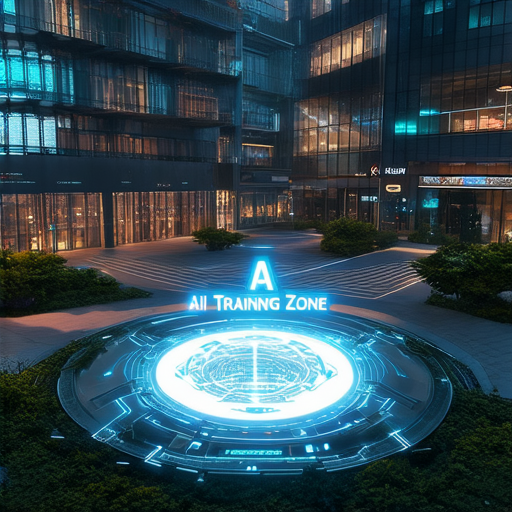}
            \includegraphics[width=.49\linewidth]{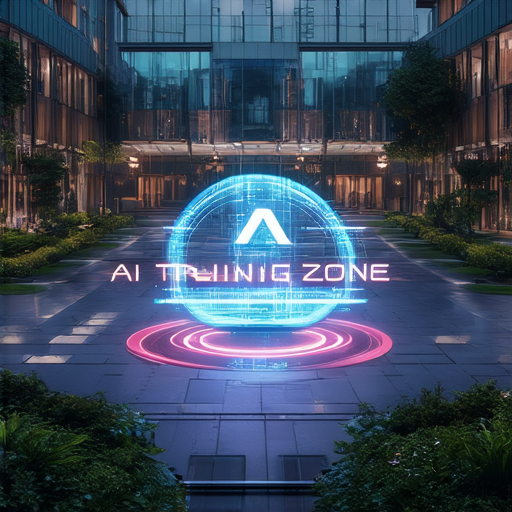}
        \end{minipage}
        \hfill
        \begin{minipage}{0.32\linewidth}
            \centering
            \includegraphics[width=.49\linewidth]{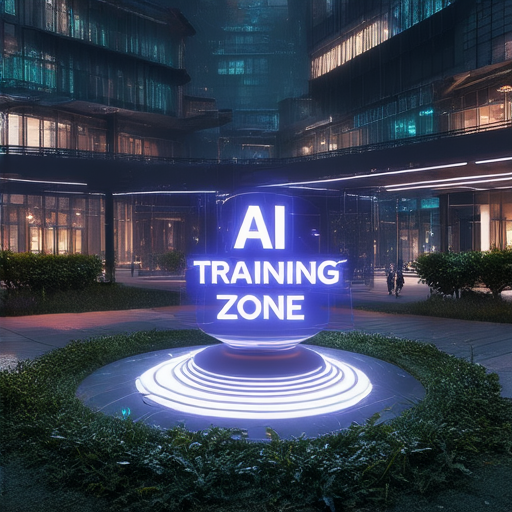}
            \includegraphics[width=.49\linewidth]{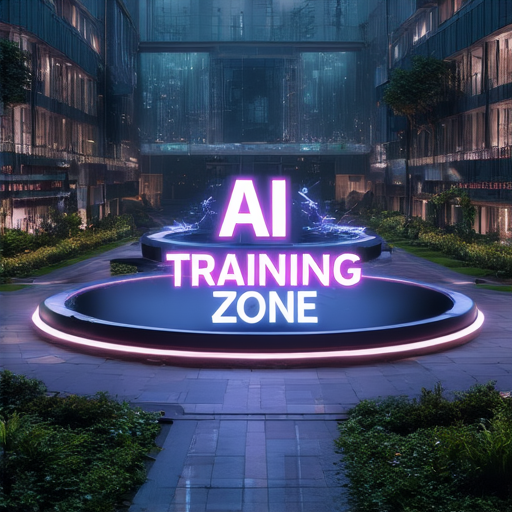}
        \end{minipage}
        \hfill
        \begin{minipage}{0.32\linewidth}
            \centering
            \includegraphics[width=.49\linewidth]{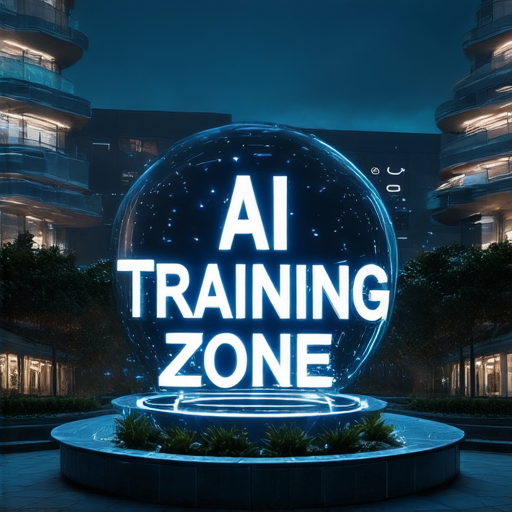}
            \includegraphics[width=.49\linewidth]{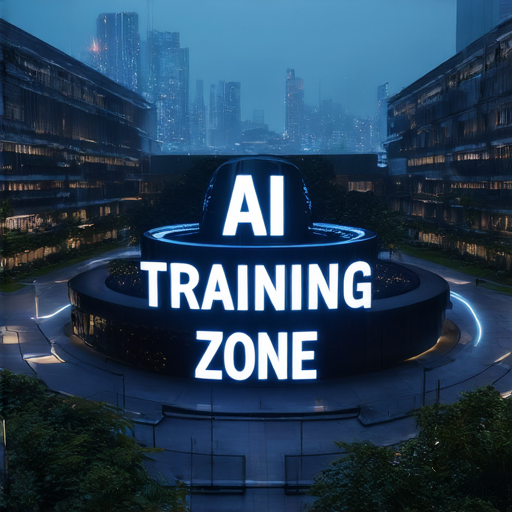}
        \end{minipage}
        \begin{minipage}{1.\linewidth}
            \smallskip
            \centering
            {\scriptsize A realistic photo of a tech campus courtyard at night, featuring a glowing {\color{red}``A I Training Zone''} hologram floating in the center, surrounded by futuristic buildings and greenery, with soft ambient lighting enhancing the futuristic atmosphere.\par}
        \end{minipage}
    \end{minipage}
    \\ \vspace{1mm}
    \begin{minipage}{0.95\linewidth}
        \centering
        \begin{minipage}{0.32\linewidth}
            \centering
            \includegraphics[width=.49\linewidth]{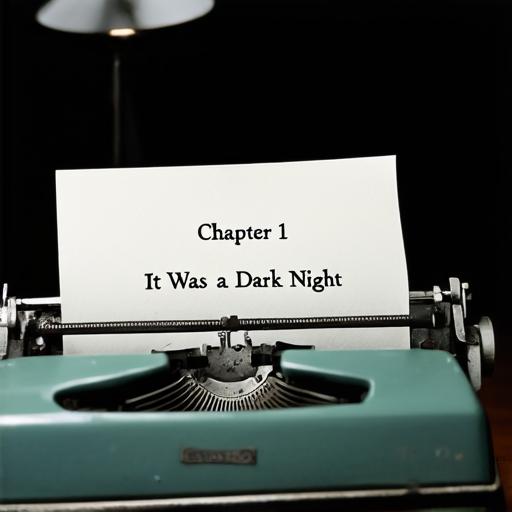}
            \includegraphics[width=.49\linewidth]{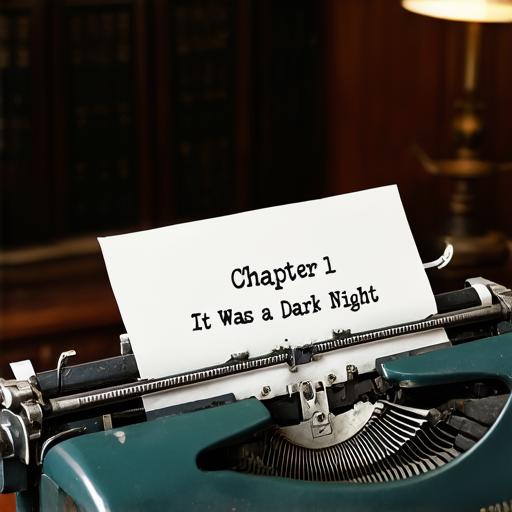}
        \end{minipage}
        \hfill
        \begin{minipage}{0.32\linewidth}
            \centering
            \includegraphics[width=.49\linewidth]{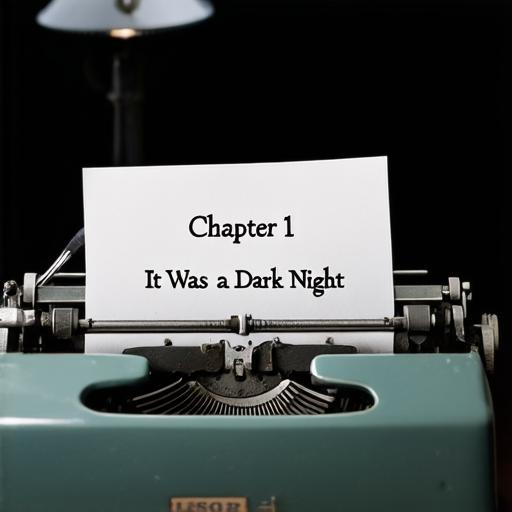}
            \includegraphics[width=.49\linewidth]{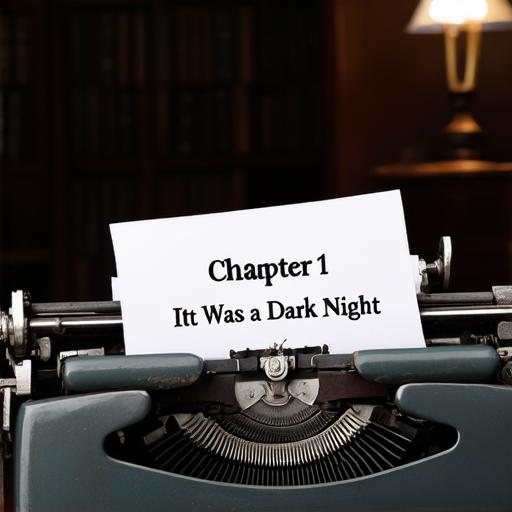}
        \end{minipage}
        \hfill
        \begin{minipage}{0.32\linewidth}
            \centering
            \includegraphics[width=.49\linewidth]{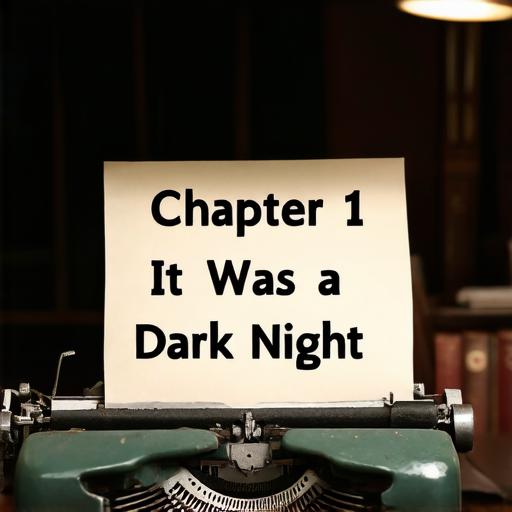}
            \includegraphics[width=.49\linewidth]{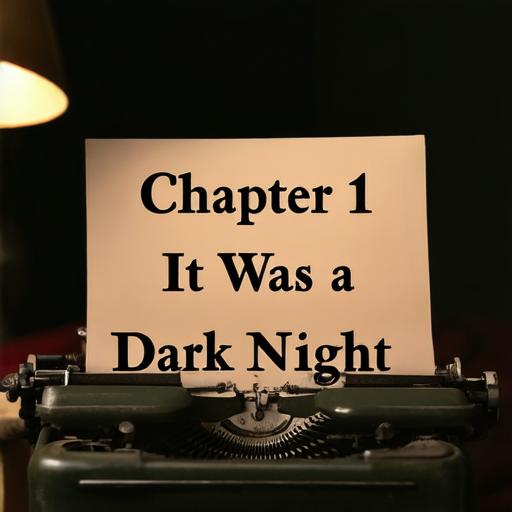}
        \end{minipage}
        \begin{minipage}{1.\linewidth}
            \smallskip
            \centering
            {\scriptsize An antique typewriter with a sheet of paper inserted, prominently displaying the typed words: {\color{red}``Chapter 1 It Was a Dark Night''}. The scene is set in a dimly lit, vintage study with a single desk lamp casting a warm glow over the typewriter.\par}
        \end{minipage}
    \end{minipage}
    \\ \vspace{1mm}
    \begin{minipage}{0.95\linewidth}
        \centering
        \begin{minipage}{0.32\linewidth}
            \centering
            \includegraphics[width=.49\linewidth]{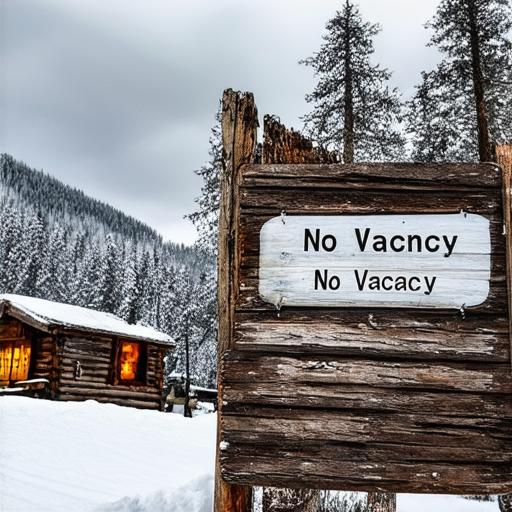}
            \includegraphics[width=.49\linewidth]{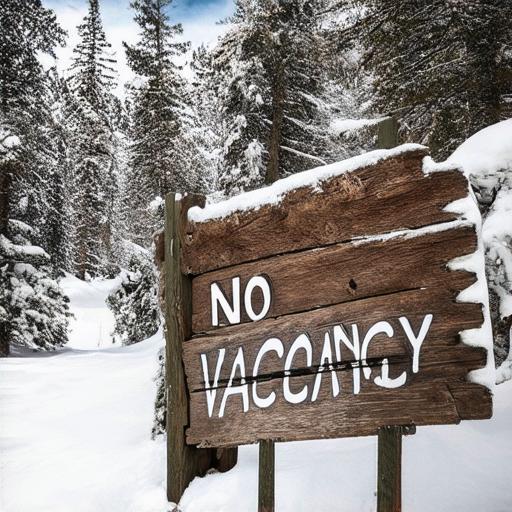}
        \end{minipage}
        \hfill
        \begin{minipage}{0.32\linewidth}
            \centering
            \includegraphics[width=.49\linewidth]{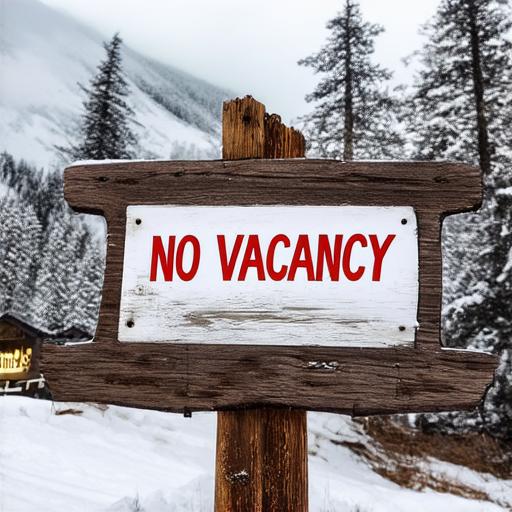}
            \includegraphics[width=.49\linewidth]{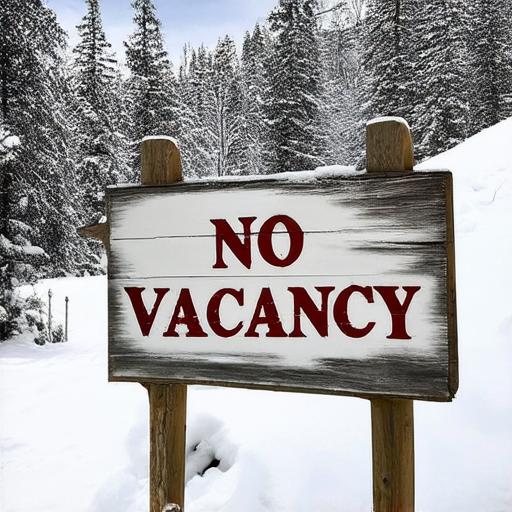}
        \end{minipage}
        \hfill
        \begin{minipage}{0.32\linewidth}
            \centering
            \includegraphics[width=.49\linewidth]{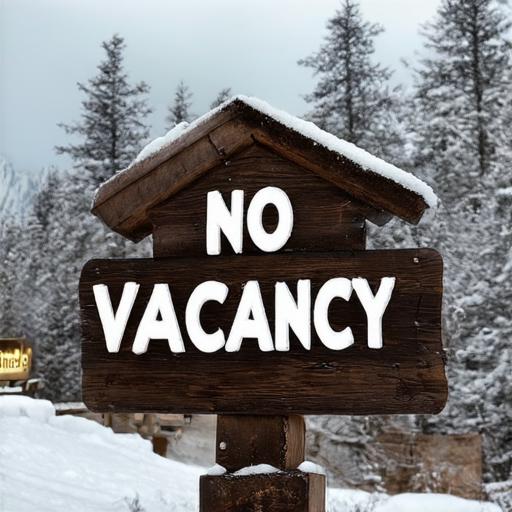}
            \includegraphics[width=.49\linewidth]{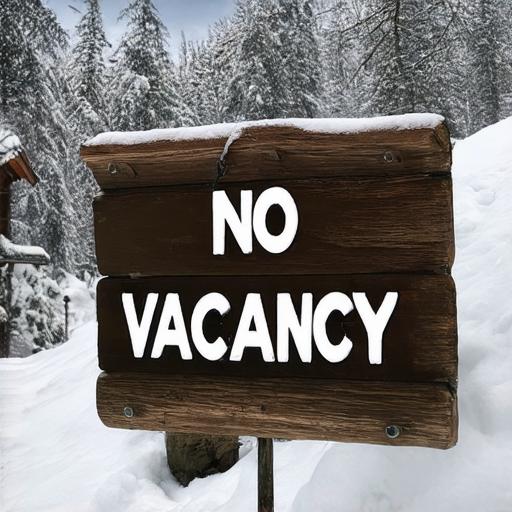}
        \end{minipage}
        \begin{minipage}{1.\linewidth}
            \smallskip
            \centering
            {\scriptsize A weathered wooden sign outside a mountain lodge reading {\color{red}``No Vacancy''}, with snow-covered trees and a cloudy sky in the background.\par}
        \end{minipage}
    \end{minipage}
    \\ \vspace{1mm}
    \begin{minipage}{0.95\linewidth}
        \centering
        \begin{minipage}{0.32\linewidth}
            \centering
            \includegraphics[width=.49\linewidth]{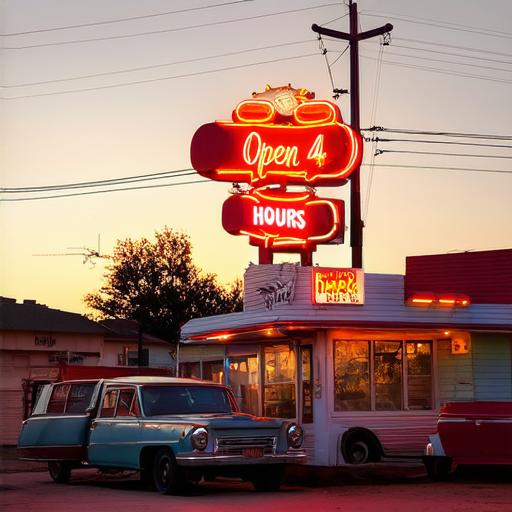}
            \includegraphics[width=.49\linewidth]{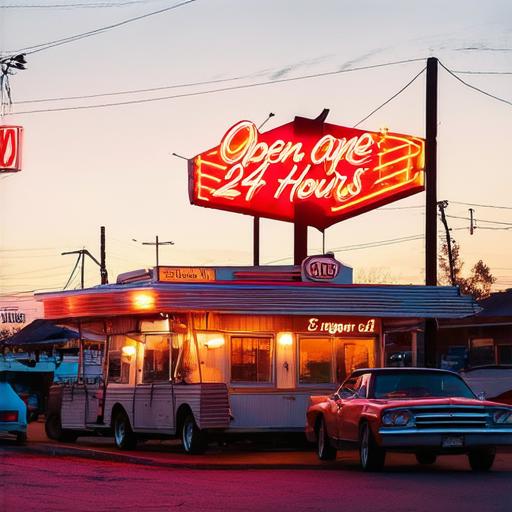}
        \end{minipage}
        \hfill
        \begin{minipage}{0.32\linewidth}
            \centering
            \includegraphics[width=.49\linewidth]{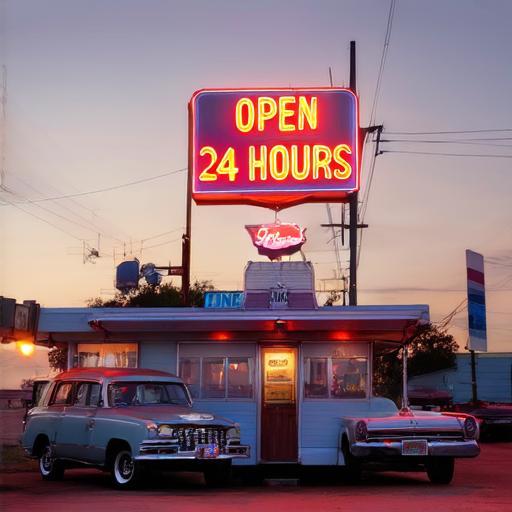}
            \includegraphics[width=.49\linewidth]{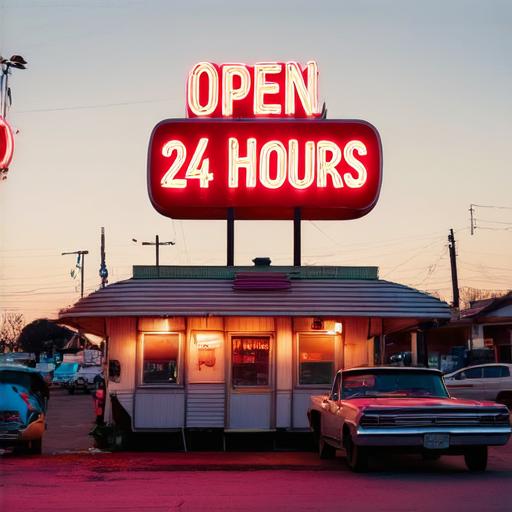}
        \end{minipage}
        \hfill
        \begin{minipage}{0.32\linewidth}
            \centering
            \includegraphics[width=.49\linewidth]{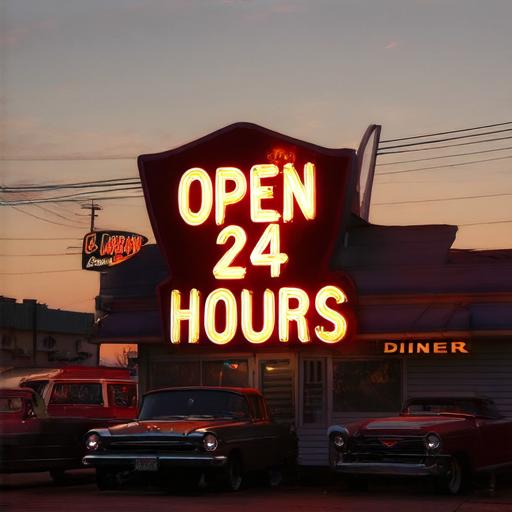}
            \includegraphics[width=.49\linewidth]{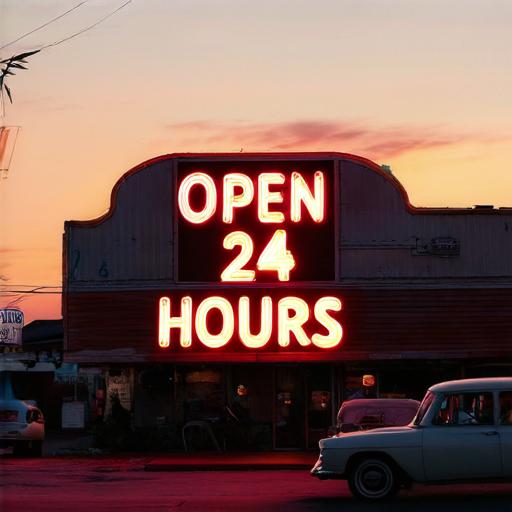}
        \end{minipage}
        \begin{minipage}{1.\linewidth}
            \smallskip
            \centering
            {\scriptsize A vintage roadside diner at sunset, with a glowing neon sign that reads {\color{red}``Open 24 Hours''}, cars parked nearby and a nostalgic atmosphere.\par}
        \end{minipage}
    \end{minipage}
    \caption{Illustration of the generated samples of different models in text rendering.}
    \label{fig:ocr-samples-comp}
\end{figure}